\documentclass{svproc}

\usepackage{graphicx}
\graphicspath{./images}
\usepackage{amsmath}
\usepackage{xcolor}
\usepackage{amsfonts}
\usepackage{balance}
\usepackage{wrapfig}
\usepackage{booktabs}
\usepackage{diagbox}
\usepackage[bookmarks=true, backref=page]{hyperref}
\usepackage[resetlabels]{multibib}
\newcites{App}{Appendix References}

\usepackage{url}

\usepackage{cite}
\usepackage{dsfont}
\usepackage{algorithm}
\usepackage{algpseudocode}
\usepackage{mathtools}

\usepackage{todonotes}

\usepackage{subcaption}
\usetikzlibrary{backgrounds}

\DeclareRobustCommand{\1}[1]{\text{\usefont{U}{bbold}{m}{n}1}{#1}}

\newcommand{\specialword}[1]{%
\textit{#1}%
}

\newcommand\policy{\pi}
\newcommand\opt{^*}
\newcommand\policyX{\pi^\mathcal{X}}
\newcommand\policyT{\pi^T}

\begin{document}
\title{Hierarchical Reinforcement Learning under Mixed Observability}
%
%
\author{Hai Nguyen\inst{1*} \and
Zhihan Yang\inst{2*} \and
Andrea Baisero\inst{1}, Xiao Ma\inst{3}, Robert Platt\inst{1\dagger} \and Christopher Amato\inst{1\dagger} \\ $*$\footnotesize{ Equal contribution \quad $\dagger$ Equal Advising}}
\authorrunning{Nguyen, Yang, Baisero, Ma, Platt, Amato}
%
\institute{Khoury College of Computer Sciences, Northeastern University, Boston, MA, USA \and
Calerton College, Northfield, MN, USA \and National University of Singapore, Singapore\\
\email{nguyen.hai1@northeastern.edu}}
\maketitle              
\begin{abstract}
The framework of mixed observable Markov decision processes (MOMDP) models many robotic domains in which some state variables are fully observable while others are not. In this work, we identify a significant subclass of MOMDPs defined by how actions influence the fully observable components of the state and how those, in turn, influence the partially observable components and the rewards. This unique property allows for a two-level hierarchical approach we call HIerarchical Reinforcement Learning under Mixed Observability (HILMO), which restricts partial observability to the top level while the bottom level remains fully observable, enabling higher learning efficiency. The top level produces desired goals to be reached by the bottom level until the task is solved. We further develop theoretical guarantees to show that our approach can achieve optimal and quasi-optimal behavior under mild assumptions. Empirical results on long-horizon continuous control tasks demonstrate the efficacy and efficiency of our approach in terms of improved success rate, sample efficiency, and wall-clock training time. We also deploy policies learned in simulation on a real robot.
\end{abstract}

\setlength\intextsep{0pt}

\section{Introduction}

Many robotic domains feature a state space that factorizes into high and low observability subspaces, in which actions primarily influence the high observability components of the state. For example, robot navigation with unknown dynamics and noisy sensors to an unknown dynamic target~\cite{chung2011search} (Fig.~\ref{fig:intro}a), or robot manipulation to reach an unknown target pose, \emph{e.g.}, find an object in cluttered and occluded environments~\cite{xiao2019online} (Fig.~\ref{fig:intro}b), grasp under uncertainty~\cite{hsiao2007grasping}, or collaborate with humans with partially observed human factors (\emph{e.g.}, trust~\cite{chen2018planning}, preferences~\cite{wang2016impact}, or goals~\cite{nikolaidis2017human}). In these examples, state variables such as a robot's position or an arm's pose can be measured with high accuracy (\emph{e.g.}, using GPS signals and sensors) than partially observable variables relative to the task and are often assumed fully observable. Moreover, actions directly influence the fully observable state components and indirectly affect partially observable ones. For example, in Fig.~\ref{fig:intro}, navigation actions can lead the mobile robot to different locations that contain useful information to reach the destination; sequences of poses help the robot arm open boxes to examine which contains the object.

\begin{figure}[t]
    \centering
    \includegraphics[width=0.85\linewidth]{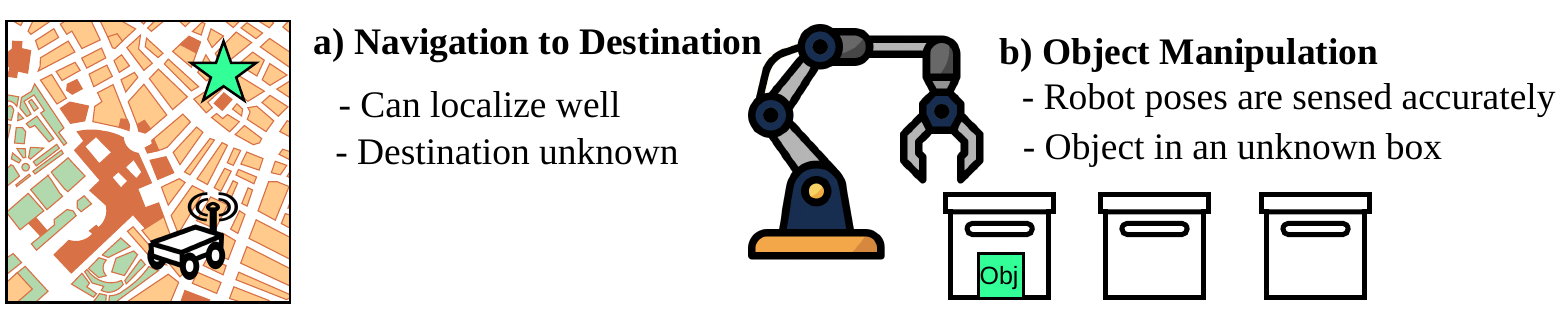}
    \caption{Examples of partially observable domains of our interests.}
    \label{fig:intro}
\end{figure}

Our contributions are as follows. First, we formulate such tasks using the framework of mixed observability Markov decision process (MOMDP) \cite{ong2009pomdps} and define motion-based MOMDPs (MOB-MOMDPs), a subclass of MOMDPs which makes mild assumptions concerning dynamics and tasks. Second, we introduce a hierarchical solution to solve MOB-MOMDPs, consisting of a high-level policy with partial observability and a low-level policy with full observability. In our agent, the top policy uses high-level observations to compute the bottom policy's goals, which are desired points in the fully observable space. When the bottom policy achieves a given goal or times out, emitted observations are used to produce the next high-level observation for selecting the next goal, and so on. Our approach can potentially offer efficient learning by breaking a long-horizon task into easier-to-learn subtasks, which enjoy full observability. Moreover, a hierarchical approach would explore more efficiently when rewards are sparse, thanks to high-level actions. In both theory and empirical experimentation, we show that our proposed hierarchical approach can achieve optimal or quasi-optimal behavior in MOB-MOMDPs with sufficiently low stochasticity constraints.

We demonstrate the benefits of our approach on long-horizon simulated continuous control domains with sparse rewards. Such domains are challenging for many non-hierarchical POMDP methods~\cite{han2019variational, Ma2020Discriminative}. In contrast, our hierarchical agent achieves higher success rates with excellent efficiency in training samples and wall-clock training time. Further, our robot experiments show that a learned policy could be effectively deployed in the real world.

\section{Background}
\label{sect:background}
In this section, we will first go through the background of a goal-conditioned MDP (which will be solved by the bottom level), then the frameworks of partially observable Markov decision processes (POMDP) and MOMDP. We conclude by the hierarchical reinforcement learning algorithm that our approach builds upon.

A \textbf{goal-conditioned MDP} is defined by a tuple $(\mathcal{S}, \mathcal{A}, \mathcal{G}, T, R, \gamma)$, where $S$ is the state space, $\mathcal{A}$ is the action space, $\mathcal{G}$ is the goal space, $T: \mathcal{S} \times \mathcal{A} \rightarrow \mathcal{S}$ is the transition function, $R: \mathcal{S} \times \mathcal{A} \times \mathcal{G} \rightarrow \mathbb{R}$ is the reward function, and $\gamma \in [0, 1)$ is the discount factor. The objective is to find a goal-conditioned policy $\policy\colon \mathcal{S}\times\mathcal{G}\to\Delta\mathcal{A}$ which maximizes the return $\mathbb{E}_\pi[\sum_{t=0}^{\infty} \gamma^t r_t]$, where $r_t$ is the reward at timestep $t$.

A \textbf{POMDP}~\cite{astrom1965optimal} is specified by a tuple
$(\mathcal{S}, \mathcal{A}, T, R, \Omega, O, \gamma)$, where $(\mathcal{S}$, $\mathcal{A}$, $T$, $R$, $\gamma)$ are the same as in a goal-conditioned MDP. Instead of directly observing the state $s$, the agent only observes $o \in \Omega$ after taking an action $a$ and reaching state $s'$ governed by the observation function $O (s', a, o) = p(o \mid  s', a)$.
The goal is to find a policy $\pi$ that maximizes the expected discounted return defined as
$\mathbb{E}_\pi[\sum_{t=0}^{\infty} \gamma^t r_t]$. 
To take an optimal action at timestep $t$, an agent often must condition its policy on the entire action-observation history $h_t = (o_{\leq t}, a_{<t}) \in \mathcal{H}_t$ that it has seen so far. However, the size of the history $\mathcal{H}_t$ grows exponentially with $t$. Therefore, a recurrent neural network (RNN) is often used to summarize $h_t$ with its fixed-sized hidden state.

A \textbf{MOMDP}~\cite{ong2009pomdps} is a POMDP in which a state can be decomposed as $s=(x,y)$, where $x\in\mathcal{X}$ is fully observable and $y\in\mathcal{Y}$ is partially observable. Since $x$ is fully observable, an observation $o$ can be decomposed as $o=(x, z) \in \Omega$ where $z$ is the remaining component of $o$. The observation function
\begin{align}
O(s', a, o) = p(o\mid s',a) = p(x, z\mid x', y',a) = \1{\left[x=x'\right]} p(z\mid x', y',a) \,
\end{align}
specifies which observation the agent gets after it took action $a$ and reached state $s'=(x',y')$ with $\1{}$ denoting the indicator function. The transition function $T(s, a, s') = p(s' \mid s, a) =  p(x', y'\mid x, y, a)$ for $x, x' \in \mathcal{X}$ and $y, y' \in \mathcal{Y}$ specifies the probabilities of reaching a state $(x',y')$ after taking an action $a$ in state $(x,y)$. $T(s, a, s')$ can be decomposed as 
\begin{align}
T(s,a,s') = p(s'\mid s, a) = p(x', y' \mid x, y, a) = p(x'\mid x,y,a)p(y'\mid x,y,a,x').
\label{eq:transition}
\end{align}
Let $T^\mathcal{X}(x,y,a,x') = p(x'\mid x,y,a)$ and $T^\mathcal{Y}(x,y,a,x',y') = p(y'\mid x,y,a,x')$, the tuple $(\mathcal{X}, \mathcal{Y}, \mathcal{A}, T^\mathcal{X}, T^\mathcal{Y}, R, \Omega, O, \gamma)$ formally defines a MOMDP.

\textbf{Hierarchical Actor-Critic (HAC)}~\cite{levy2017learning} is an MDP hierarchical agent, in which an action from a non-base level is a goal for the policy at the level right below it, and the base policy will directly interact with the environment. The policies in each level are trained in an off-policy manner using replay buffers, one for each level. We build our agent upon a two-level HAC agent and utilize different techniques in HAC to stabilize the training at the top level and learn effectively under sparse rewards at the bottom. More details are in Section \ref{sect:HILMO}.

\section{Motion-Based MOMDPs}
\label{sect:mob-momdp}

We define motion-based MOMDPs (MOB-MOMDPs) as MOMDPs which satisfy the following additional factorization and independence assumptions,
\begin{align}
p(s'\mid s, a) &= p(x'\mid x, a) p(y'\mid x, y, x') \,, \\
p(o\mid s', a) &= p(o\mid s') \,, \\
R(s, a) &= R(s) \,.
\end{align}
\noindent In other words, a) the fully observable component $x$ of the state satisfies the Markov property without depending on the partially observable component $y$ of the state, b) both the partially observable component $y$ of the state and the observed component $z$ are conditionally independent on the action $a$ (when conditioned on the fully observable component $x$ of the state), and c) the task is encoded by a reward function which exclusively depends on the reached states, and not the actions taken.  We refer to MOB-MOMDPs, which have deterministic (stochastic) $T^\mathcal{X}$ as \emph{deterministic} (\emph{stochastic}) MOB-MOMDPs;  note that this does not refer to the stochasticity of $T^\mathcal{Y}$, which remains unconstrained.

In MOB-MOMDPs, actions only have a direct influence on the resulting $x$ trajectories, while their influence on the $y$, $z$, and reward trajectories is indirect through $x$.  MOB-MOMDPs include (but are not limited to) navigation tasks where $x$ represents the fully observable pose of the agent in the environment, while $y$ represents other partially observable information about the environment and task.  In such navigation MOB-MOMDPs, actions relate to the \emph{motion} of the agent, and it is exclusively through such motion that the agent is able to interact with the environment, gather information, and complete the task.  Although not all MOB-MOMDPs intrinsically represent navigation tasks, we will use the imagery of navigation tasks as a useful analogy to simplify the way we discuss and analyze MOB-MOMDPs and the respective learning algorithms.  Therefore, we reinterpret general MOB-MOMDPs as navigation tasks where $x$ figuratively represents the agent's pose, $a$ the movements that allow the agent to change its pose, and $y$ as any other partially observable aspect concerning the task.

Using such analogy, and because actions exclusively influence the environment through the resulting agent pose, it is possible to abstract ``motion''-based control (\emph{i.e.}, based on the actions which move the agent) as ``pose''-based control (\emph{i.e.}, based on the poses which the agent should reach in order to gain information or complete the task).  ``Pose''-based control is executed not by choosing how the agent should move (action $a$), but rather where it should move to (pose $x'$).  Such abstraction is the inspiration for a flavor of hierarchical reinforcement learning specifically suited to solve MOB-MOMDPs.

\section{Hierarchical Reinforcement Learning under Mixed Observability}
\label{sect:HILMO}

We first give an overview of our hierarchical method, and how each hierarchy layer is trained.  Then we provide an optimality analysis of the approach.

\subsection{Approach}
\textbf{Overview.} 
As shown in Fig.~\ref{fig:method}, HIerarchical reinforcement Learning under Mixed Observability (HILMO) makes decisions through a two-level hierarchy. The top-level policy is a recurrent module (symbolized by an arrow pointing to itself) which takes in a top-level observation $o_t^T$ (a summary of several past primitive observations $o \in \Omega$) to produce a goal $x^g_t \in \mathcal{X}$ being the desired value for $x_t$. Then the memoryless bottom-level policy selects an action $a_t$ using $x_t$ (extracted from $o_t = (x_t, z_t)$) and $x^g_t$. Notice this also covers the case when we want to set the goal only for a subspace of $\mathcal{X}$. For instance, although $x$ of the mobile robot in Fig.~\ref{fig:intro} might include both positions and velocities, we might just want it to successfully reach a certain position regardless of its velocity. 

\begin{figure}[t]
    \centering
    \includegraphics[width=0.98\linewidth]{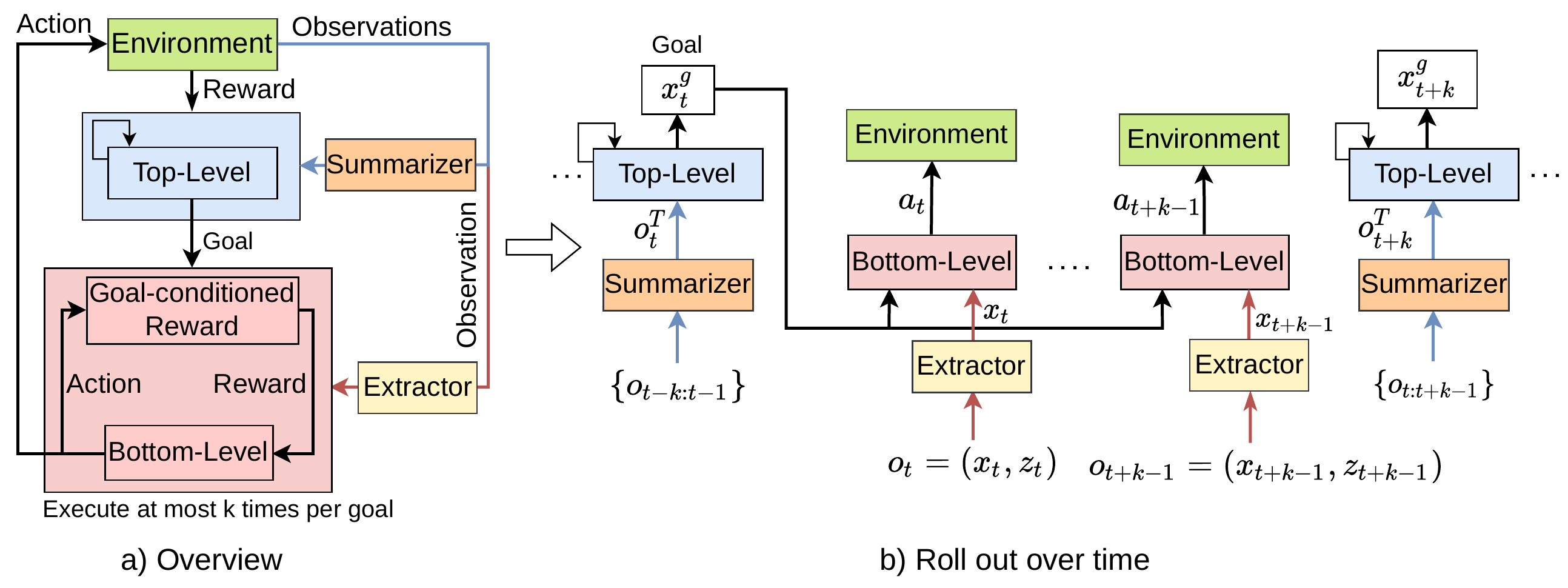}
    \caption{Our two-level hierarchical agent. A \textit{memory-based} top-level policy looks at a summary of several past observations to select a desired state (goal) $x^g_t$ for the bottom-level policy. A \textit{memoryless} bottom policy then looks at $x_t$ (a component of observation $o_t = (x_t, z_t)$) and the goal state $x^g_t$ to produce at most $k$ primitive actions to achieve $x^g_t$, emitting a new sequence of observations. Next, these observations will be fed to a summarizer to create a new high-level observation for the top policy to select a new goal. }
    \label{fig:method}
\end{figure}

The goal $x^g_t$ remains unchanged until $x^g_t$ is achieved or $k$ bottom-level actions have been performed (for clean notations, from now, we assume that the bottom episode will \textit{always} last $k$ timesteps). When the bottom level finishes acting, $k$ observations $o_{t:t+k-1}$ (\emph{i.e.}, $o_t, \dots, o_{t+k-1}$) are emitted. These observations will be fed to a summarizer to create the next top-level observation $o^T_{t+k}$ for the top policy to choose the next goal. In this hierarchy, the top level acts at a higher temporal resolution than the bottom level. For a complete algorithm of HILMO, please refer to Appendix \ref{app:HILMO}.

\textbf{Bottom-Level. } A bottom-level goal-conditioned MDP $\mathcal{M}^\mathcal{X}$ is specified by $(\mathcal{X}, \mathcal{A}, \mathcal{G} = \mathcal{X}, T^\mathcal{X}, R^\mathcal{X}, \gamma)$. The reward function $R^\mathcal{X}$ is defined for each goal $x^g \in \mathcal{X}$ as $R^\mathcal{X} (x',x^g) = -\1 [d(x', x^g) \geq \epsilon]$, where $d$ is some given distance metric in $\mathcal{X}$ and $\epsilon$ is a small reaching threshold. The transition function $T^\mathcal{X}(x,a,x') = p(x'\mid x,a)$ specifies the probabilities of reaching $x'$ after taking action $a$ in $x$. A goal-conditioned policy $\pi^\mathcal{X}(a\mid x, x^g)$ that solves $\mathcal{M}^\mathcal{X}$ will maximize the discounted cumulative reward $\sum_{t'=t}^{t+k-1} \gamma^{t'-t}R^\mathcal{X}(x_{t'}, x^g)$, where $x_t$ is the starting state. The input $x$ of $\pi^\mathcal{X}$ is from an extractor that extracts $x$ from $o=(x,z)$.

\textbf{Top-Level. }A top-level POMDP $\mathcal{P}^T$ is specified by $(\mathcal{S}, \mathcal{A}^T = \mathcal{X}, T^T, R^T, \Omega^T,\allowbreak O^T, \gamma)$. In particular, its action space $\mathcal{X}$ is the goal space in $\mathcal{M}^\mathcal{X}$, therefore a \textit{top-level} policy $\pi^T(\cdot\mid o^T)$ that solves $\mathcal{P}^T$ will output a desired state $x^g$ to be reached by $\pi^\mathcal{X}$. The top-level transition function can be specified as a \textit{multi-time model}~\cite{precup1997multi} that describes a multi-step policy taking in the goal $a^T$ from state $s$
\begin{equation}
T^T(s, a^T,s') = \sum_{m=1}^{k} p(s',m\mid s, a^T) \,,
\end{equation}
where $p(s',m\mid s, a^T)$ is the probability that the bottom policy $\pi^\mathcal{X}$ terminates at state $s'$ after exactly $m$ primitive actions when it acts to achieve the goal $a^T$ starting from state $s$. Unlike $\pi^\mathcal{X}$, the objective of $\pi^T$ is to optimize the discounted cumulative reward $\sum_{t=0;t+=k}^\infty \gamma^{t/k}R^T(s_{t}, a^T_{t})$ where each $R^T(s_t,a_t^T)$ is the expected accumulated environment rewards when $\pi^X$ acts for $k$ timesteps to achieve the goal $a_t^T$ starting at state $s_t$. Top-level observations are defined as the sequence of actions and observations obtained by the low-level policy until control is given back to the high-level policy, e.g., assuming that at time-step $t$ the high-level policy chooses goal $x^g=a^T$, and that the low-level policy interacts with the environment for $k$ timesteps by choosing actions $(a_t, \ldots, a_{t+k-1})$ and receiving observations $(o_t, \ldots, o_{t+k-1})$, then $o^T = (a_t, o_t, \ldots, a_{t+k-1}, o_{t+k-1})$.
%


\subsection{Bottom-level Policy Learning}

Because the bottom policy acts in a fully observable system with sparse rewards, we learn it using transitions sampled from a replay buffer using goal relabeling~\cite{andrychowicz2017hindsight}.

\vspace{5pt}
\noindent \textbf{Goal Relabeling} is a commonly used and powerful technique for learning under sparse rewards. We utilize the technique to replace unmet goals in past transitions with ones met in hindsight, creating positive learning signals as goals are met. Similarly, HAC uses goal relabeling to create hindsight goal transitions (HGT). Specifically, given a bottom-level transition $(x, a, x', r, x^g)$ that did not reach $x^g$ but reached $x'^g$ instead, a modified transition $(x, a, x', r', x'^g)$ where $r'$ is the new reward associated with $x'^g$ will be used for training. In contrast, if $x^g$ is reached when the transition ends, the original transition will be used. 

\vspace{5pt}
\noindent \textbf{Learning Algorithm.} Adopting HAC's choice, we use a version of DDPG \cite{lillicrap2015continuous} without target networks. We also experimented with other learning algorithms such as TD3 \cite{fujimoto2018addressing} and SAC~\cite{haarnoja2018soft}, but none outperformed DDPG. Please see Appendix \ref{app:ddpg} for more details about our implementation of DDPG.

\subsection{Top-Level Policy Learning}
\label{ref:top-learn}

The top policy is trained off-policy using samples from a replay buffer of top-level episodes. This differs from HAC's memoryless top level, which can be trained using transitions. Here we introduce methods to create top-level observations and forming training episodes before describing the learning algorithm.

\noindent \textbf{Creating Top-Level Observations. } In practice, the summarizer applies an operator $\mathcal{F}$ to $k$ low-level observations to create a high-level observation. In this perspective, a high-level observation can be considered as an implication of temporally extended perception. Here, we consider three options for $\mathcal{F}$.

\noindent \underline{Full.} Concatenating all $k$ previous observations to create a top-level observation. If the bottom policy finishes before $k$ timesteps, we simply pad zeros to have $k$ observations.

\noindent \underline{Final.} Using the final ($k$-th) observation as a top-level observation. This approach is commonly used in MDP hierarchical agents such as HAC and HIRO~\cite{HIRO}.

\noindent \underline{Recurrent.} Using a separate and learnable recurrent layer to summarize all $k$ previous observations with the final hidden state being a top-level observation.

\vspace{5pt}
\noindent\textbf{Creating Stationary Training Episodes. } Given a top-level episode in with any goal unachieved, using it directly to learn the top policy will cause non-stationarity. For instance, given the same $(o^T, a^T)$, the bottom policy can achieve other undesired goals, leading to multiple possibilities of $(o'^T, r^T)$. This makes learning $Q(o^T, a^T)$ at the top level non-stationary.

\noindent \underline{Creating Stationary Transitions:} Stationary transitions are created using hindsight action transitions (HAT)~\cite{levy2017learning}, which are modified transitions \textit{as if} the bottom policy already converges and always achieves its goals. Specifically, given a transition $(o^T, a^T_{\text{unmet}}, o'^{T}, r^T)$ with an unmet action (goal) $a^T_{\text{unmet}}$, the transition $(o^{T}, a^T_{\text{met-by-bottom}}, o'^{T}, r^T)$ will instead be used, where $a^T_{\text{met-by-bottom}}$ is the goal reached by the bottom policy.

\noindent \underline{Penalizing Unachieved Goals:} Producing unrealistic goals that are unreachable by the current bottom policy should be discouraged. Therefore, whenever the top policy produces an unachieved goal for the bottom policy, it will be penalized by a negative reward with some probability. Specifically, given a transition $(o^T, a^T_{\text{unmet}}, o'^T, r^T)$, the transition $(o^T, a^T_{\text{unmet}}, o'^T, -H^T)$ will be used, where $H^T$ is the time horizon of the top policy. These transitions are known as subgoal testing transitions (STT)~\cite{levy2017learning}.

\begin{wrapfigure}{R}{0.5\linewidth}
    \centering
    \includegraphics[width=0.95\linewidth]{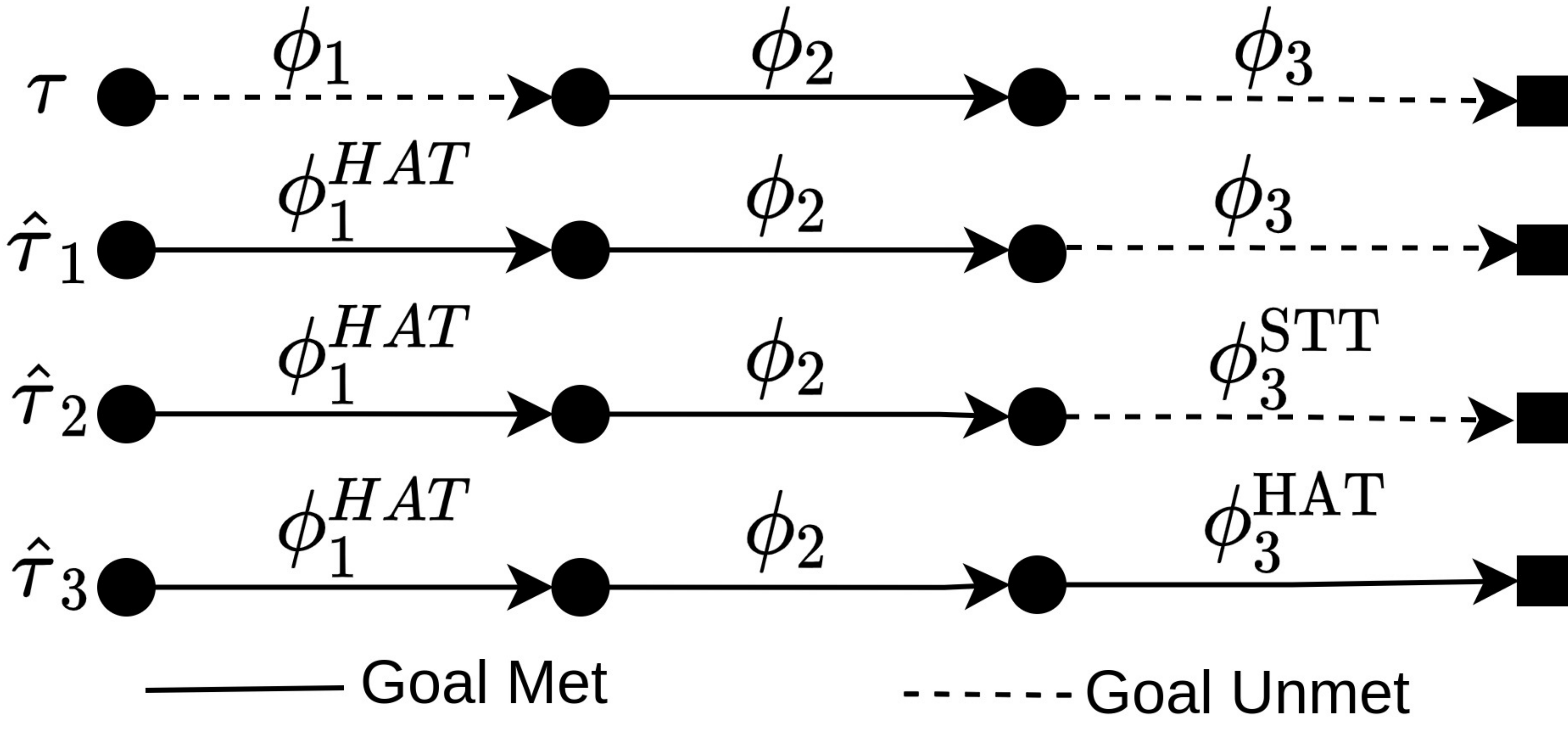}
    \caption{Creating training episodes for $\policyT$.}
    \label{fig:transition}
\end{wrapfigure}
\noindent \underline{Forming Stationary Episodes:} While HATs or STTs alone are sufficient to train a memoryless HAC agent, we must combine those transitions to form episodes for training our recurrent top policy. Fig.~\ref{fig:transition} illustrates a method transforming a non-stationary episode into its stationary versions. Here a top-level non-stationary episode $\tau$ comprises three transitions $\phi_1, \phi_2$, and $\phi_3$ in which a goal is only achieved (by the bottom policy) in transition $\phi_2$. We a) replace $\phi_1$ with $\phi_1^{\text{HAT}}$ to create a stationary episode $\hat{\tau}_1$, b) use HAT for $\phi_1$ and STT for $\phi_3$ to create episode $\hat{\tau}_2$, or c) use HAT for $\phi_1$ and $\phi_3$ to create episode $\hat{\tau}_3$. Finally, three stationary episodes $\hat{\tau}_1, \hat{\tau}_2$, and $\hat{\tau}_3$ are used to train $\pi^T$.

\noindent \textbf{Learning Algorithm.} We use Recurrent Deterministic Policy Gradient (RDPG) \cite{heess2015memory} to learn the top-level policy with an LSTM~\cite{hochreiter1997long} recurrent component. For the algorithm and the implementation details, see Appendix \ref{app:rdpg}. We also experiment with Gated Recurrent Unit (GRU) \cite{cho2014learning} and a recurrent version of TD3~\cite{fujimoto2018addressing} instead of RDPG (see Appendix \ref{app:design_choice} for a performance comparison).

\subsection{Optimality Analysis}

In this section, we analyze the optimality of the HILMO approach in MOB-MOMDPs, \emph{i.e.}, whether the best HILMO policies are also guaranteed to achieve optimal MOB-MOMDP behavior.  As it turns out, the amount of stochasticity of the MOB-MOMDP (as defined in Section~\ref{sect:mob-momdp}, \emph{i.e.}, only concerning $T^\mathcal{X}$) becomes a determining factor on the optimality of the HILMO approach.  Initially, we will assume \emph{deterministic} MOB-MOMDPs, and show that optimality is guaranteed.

\begin{theorem}\label{thm:optimality}
For \emph{deterministic} MOB-MOMDPs and sufficiently large $\gamma$, the optimal HILMO policies are optimal or quasi-optimal for the MOB-MOMDP.
\end{theorem}

\begin{proof}
We first assume $\gamma=1$, and show by construction that optimal MOB-MOMDP policies can be represented as HILMO policies, and then show that no other HILMO policies are preferable according to the HILMO criteria.

Consider, without loss of generality, a deterministic optimal policy $\policy\opt$ for the \emph{deterministic} MOB-MOMDP.  Next, we show that we can construct a HILMO policy which exhibits the same optimal behavior.  Given any history $h$ and its associated fully observable pose $x$, the optimal policy selects action $a=\policy\opt(h)$ which causes a transition into pose $x'$.  Due to the deterministic assumptions on $T^\mathcal{X}$ and $\policy\opt$, each $h$ is associated with a unique resulting transition $(h, x, a, x')$.
Consider the HILMO policy $(\policyT, \policyX)$ constructed such that for each such tuple $(h, x, a, x')$, $\policyT(h) = x'$ and $\policyX(x, x') = a$, while all other actions $\policyX(x, g)$ can be chosen to achieve the shortest path between $x$ and $g$ to satisfy the low-level MDP optimality criterion.  Such HILMO policy exhibits the exact same behavior as the optimal policy $\policy\opt$, \emph{i.e.}, for all histories $h$ the equality $\policyX(x, \policyT(h)) = \policy\opt(h)$ holds.
This shows that any behavior which is optimal for the control problem can be represented as a HILMO policy.
To conclude, we need to show that there is no other HILMO policy that would be preferable to the one constructed according to the HILMO criteria of optimality for the low-level and high-level policies respectively.  This is trivial for the constructed low-level policy, which effectively already finds the shortest (i.e., one-step) path between two poses $x$ and $x'$, and for the constructed high-level policy, which shares the same optimality criterion as the original MOB-MOMDP.

When $\gamma < 1$, the HILMO top-level criterion and the MOB-MOMDP criterion apply slightly different forms of discounting.  However, for sufficiently large $\gamma$, the difference is small enough to ensure that the criteria are either equivalent or approximately equivalent, resulting in optimal or quasi-optimal HILMO policies.
\end{proof}

Note that the optimal HILMO policy constructed above executes at the smallest possible temporal scale, \emph{i.e.}, the high-level policy $\policyT$ selects goal poses which are directly adjacent to the agent's current pose, and the low-level policy $\policyX$ is able to reach such goal poses in a single timestep.  However, this does not preclude the existence of other optimal HILMO policies which execute at broader temporal scales, in which the high-level policy selects goal poses which require multiple timesteps to be reached.

This analysis is limited to \emph{deterministic} MOB-MOMDPs and does not intrinsically carry over to \emph{stochastic} MOB-MOMDPs.  In 
Appendix~\ref{app:stochasticity}, we extend the analysis, providing a highly stochastic MOB-MOMDP example which demonstrates issues with the HILMO approach and a lowly stochastic MOB-MOMDP, which demonstrates that such issues are minor if the stochasticity is also minor.  We argue that because most realistic MOB-MOMDPs have a relatively low amount of stochasticity, the HILMO approach should still be able to achieve theoretically quasi-optimal performance even in mildly stochastic MOB-MOMDPs.

%
%
%
%
%
%

\section{Related Work}
This section describes prior works on hierarchical reinforcement learning (HRL).

\vspace{5pt}
\noindent \textbf{HRL for MDPs.}
For discrete action spaces, several prior works \cite{ dietterich2000hierarchical, kulkarni2016hierarchical} addressed learning hierarchically. For continuous action spaces, HAC~\cite{levy2017learning} and HIRO~\cite{HIRO} proposed a hierarchical agent consisting of policies learned jointly in an off-policy manner. HIPPO \cite{HIPPO} is an on-policy hierarchical agent but focuses more on optimizing pre-trained skills for downstream tasks.

\vspace{5pt}
\noindent \textbf{HRL for POMDPs.}
Hierarchical Suffix Memory \cite{hernandez2001hierarchical} incorporated hierarchies with memories to solve navigation tasks under perceptual aliasing. Another two-level hierarchical agent \cite{le2018deep} used hand-crafted goals with policies learned independently. HQ-Learning \cite{wiering1997hq} solved specific POMDPs using a pre-specified number of sequentially chained reactive sub-agents. \cite{steckelmacher2018reinforcement} combined hardcoded memoryless options (temporally extended actions) to solve a navigation POMDP by conditioning each option on the previous one. In contrast to these works, which use hardcoded bottom-level policies, hand-crafted goals, and have rigid structures, our hierarchical agent learns policies of all levels jointly, can handle continuous control tasks, and is supported by theoretical analysis.

\section{Experiments}
We perform experiments on continuous control tasks including two navigation and two manipulation domains implemented in MuJoCo \cite{todorov2012mujoco}. Below in Table~\ref{tab:domains}, we describe each domain and the corresponding state and observation.
\subsection{Domains}

\begin{figure}[t]
    \centering
    \includegraphics[width=1.0\linewidth]{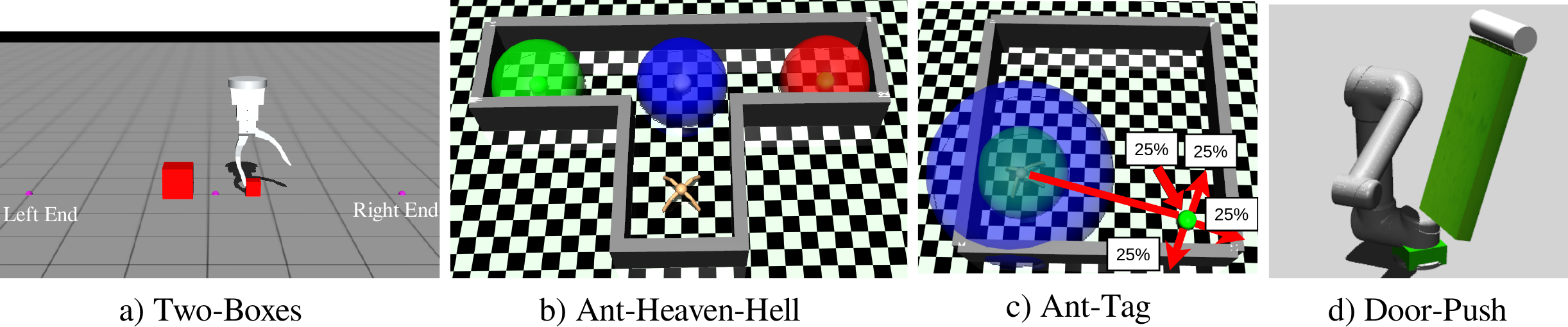}
    \caption{Four continuous control domains in MuJoCo~\cite{todorov2012mujoco} to perform experiments.}
    \label{fig:domains}
\end{figure}

\noindent \textbf{Two-Boxes. }A finger is velocity-controlled (dim$(\mathcal{A})=1$) on a 1D track to perform a dimension check of two boxes (Fig.~\ref{fig:domains}a). Since the finger is always compliant, it will be deflected from the vertical axis when it glides over a box. The agent observes the finger's position and angle (dim$(\Omega)=2$) but not the positions of the two boxes. Therefore, an optimal agent must localize \textit{both} boxes and determine their sizes using the history of angles and positions. When the two boxes have the same size, the agent must go to the right end (pink) to get a non-zero reward and otherwise to the left end. The agent receives a penalty if reaching wrong ends. An episode terminates whenever the two ends are reached, or lasts more than 100 timesteps.

\vspace{5pt}
\noindent \textbf{Ant-Heaven-Hell. }An ant with four legs (dim$(\mathcal{A})=8$) moving in a 2D T-shaped world will receive a non-zero reward by reaching a green area (heaven) that can be on the left or the right corner (Fig.~\ref{fig:domains}b) of a junction. The ant receives a penalty when entering a red area (hell). When it stays in the blue ball, it can observe heaven's side (left/right/null). Here the observation includes the joints' angles \& velocities of the four legs and the side indicator (dim$(\Omega)=30$). The ant starts randomly around the bottom corner, and an episode terminates when heaven or hell is reached, or more than 400 timesteps have passed. An optimal agent must visit the blue region to observe heaven's side, memorize the side while going to heaven, and finally goes to heaven.

\begin{table}[t]
\centering
\caption{State, goal, and observation descriptions.  $x^g$ is the goal used by the bottom-level policy in the hierarchical agents (HILMO, HILMO-O, and HAC).}
\label{tab:domains}
\begin{tabular}{p{.22\linewidth}p{.75\linewidth}}
\toprule
\textbf{Domain} & \textbf{Description} \\
\midrule
\texttt{Two-Boxes} & $x, x^g$: finger positions, $y$: desired target position, $z$: finger angle \\
\midrule
\texttt{Ant-Heaven-Hell} & $x$: joint angles and velocities, and body position, $x^g$: ant body position, $y$: heaven position, $z$: heaven position or null \\
\midrule
\texttt{Ant-Tag} & $x$: joint angles and velocities, and body position, $x^g$: ant body position, $y$: opponent position , $z$: opponent position or null \\
\midrule
\texttt{Door-Push} & $x$: joint angles and velocities, $x^g$: target joint angles, $y$: push direction, $z$: door angle \\ \bottomrule
\end{tabular}
\end{table}

\vspace{5pt}
\noindent \textbf{Ant-Tag. }The same ant now has to search and ``tag'' a moving opponent by being sufficiently close to it (having the opponent inside the green area centered at the ant in Fig.~\ref{fig:domains}c) to get a non-zero reward. Both start randomly but not too close to each other. The opponent follows a fixed stochastic policy, moving a constant distance away from the ant 75\% of the time or staying otherwise. An observation includes the joints' angles \& velocities of four legs and the 2D coordinate of the opponent (dim$(\Omega)=31$), containing the opponent's position only when it is inside the visibility (blue) area centered at the ant. An episode terminates when the opponent is tagged, or more than 400 timesteps have passed.

\vspace{5pt}
\noindent \textbf{Door-Push. } A 3-DoF gripper (dim$(\mathcal{A})=3$) in 3D must successfully push a door to receive a non-zero reward (Fig.~\ref{fig:domains}d). The door, however, can only be pushed in one direction (front-to-back or vice versa), and the correct push direction is unknown to the agent. Here the agent can observe the joints' angles and velocities and the door's angle (dim$(\Omega)=7$). Starting each episode, the door is present to the gripper, initialized with a random pose. An optimal agent must experiment to determine the correct push direction. For example, when it fails to open the door in one direction, it must go to the other side of the door while memorizing the previous push direction that did not work, not to try that again. Therefore, an optimal agent must infer the correct push direction from the history of observations. An episode terminates when the door's angle is larger than a threshold or more than 400 timesteps have passed.

\subsection{Agents}
We consider the following hierarchical \textbf{(H)} and flat \textbf{(F)} agents:
\begin{itemize}
    \item \textbf{(H)} A two-level \textbf{HILMO} (ours) agent with goals $x^g$ described in Table~\ref{tab:domains}.
    \item \textbf{(H)} A two-level \textbf{HAC}~\cite{levy2017learning} agent with the same goal as in HILMO to show that recurrence is needed for a hierarchical agent to solve our domains.
    \item \textbf{(H)} \textbf{HILMO-O} (ours) same structure and goals as HILMO, but implemented using the HIRO~\cite{HIRO} framework with off-policy corrections (see Appendix \ref{app:HILMO-O} for the algorithm) to replace HATs. Another difference is that the bottom policy of HIRO originally receives dense instead of sparse rewards, \emph{i.e.}, $R^\mathcal{X} (x',x^g) = -d(x', x^g)$, hence HGTs are not used. Plus, there is no penalty for unachieved goals by the top policy (STTs are not used either). Like HAC, HIRO is a common hierarchical baseline for continuous control.
    \item \textbf{(F)} Soft Actor-Critic (\textbf{SAC})~\cite{haarnoja2018soft} with observations instead of states to show that even a strong flat agents cannot solve our domains without memory.
    \item \textbf{(F)} Recurrent Soft Actor-Critic (\textbf{RSAC})~\cite{yang2021recurrent} is a recurrent version of SAC.
    \item \textbf{(F)} Discriminative Particle Filter Reinforcement Learning (\textbf{DPFRL})~\cite{Ma2020Discriminative} is an on-policy agent that summarizes the history using a differentiable particle filter. DPFRL is one of state-of-the-art model-free POMDP methods. 
    \item \textbf{(F)} Variational Recurrent Model (\textbf{VRM})~\cite{han2019variational} is one of state-of-the-art off-policy model-based agents. It solves POMDPs by using a recurrent variational dynamic model and an SAC agent.
    \item \textbf{(F)} \textbf{Addition baselines} are explored in Appendix \ref{app:additional_baselines}. 
\end{itemize}

\begin{figure}[t]
    \centering
    \includegraphics[width=0.9\linewidth]{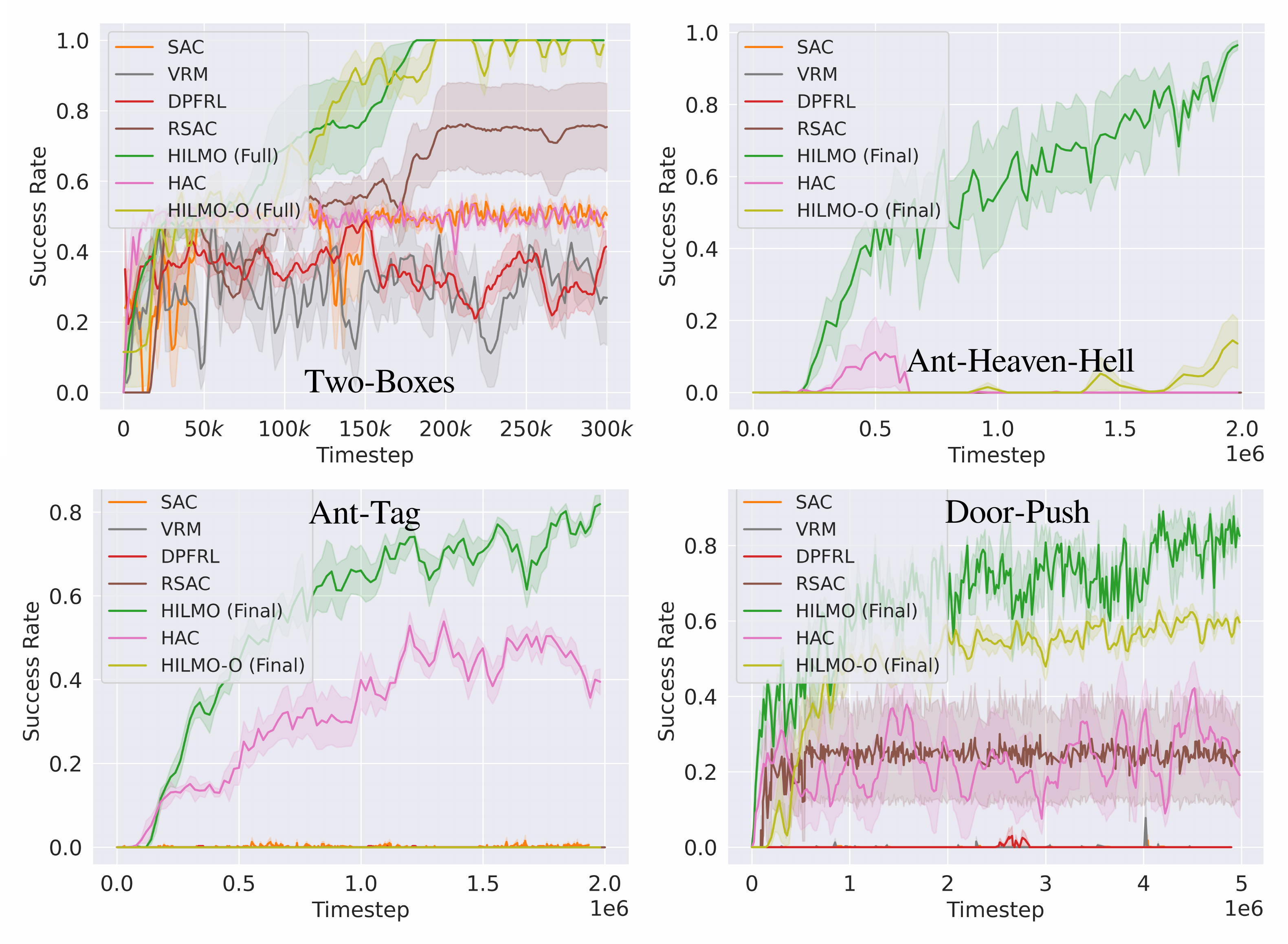}
    \caption{Success rate means and standard deviations (4 seeds).}
    \label{fig:all_success_rates}
\end{figure}

\subsection{Learning Performance}
We compare the success rates and the wall-clock training time of all agents. We alternate training and testing for all agents and compute the average success rates over 100 test episodes for every 2000 environment timesteps.  

\vspace{5pt}
\noindent \textbf{Success Rates. } For HILMO agents, we only report the best performance achieved with a specific strategy to create top-level observations (Full, Final, or Recurrent). From Fig.~\ref{fig:all_success_rates}, we can see that across all domains, HILMO consistently outperforms all flat baselines and is on par or better than HILMO-O. Moreover, HILMO is the only agent that can learn well in \texttt{Ant-Heaven-Hell} and \texttt{Ant-Tag}. 

\begin{wrapfigure}{R}{0.5\linewidth}
    \centering
    \includegraphics[width=1.0\linewidth]{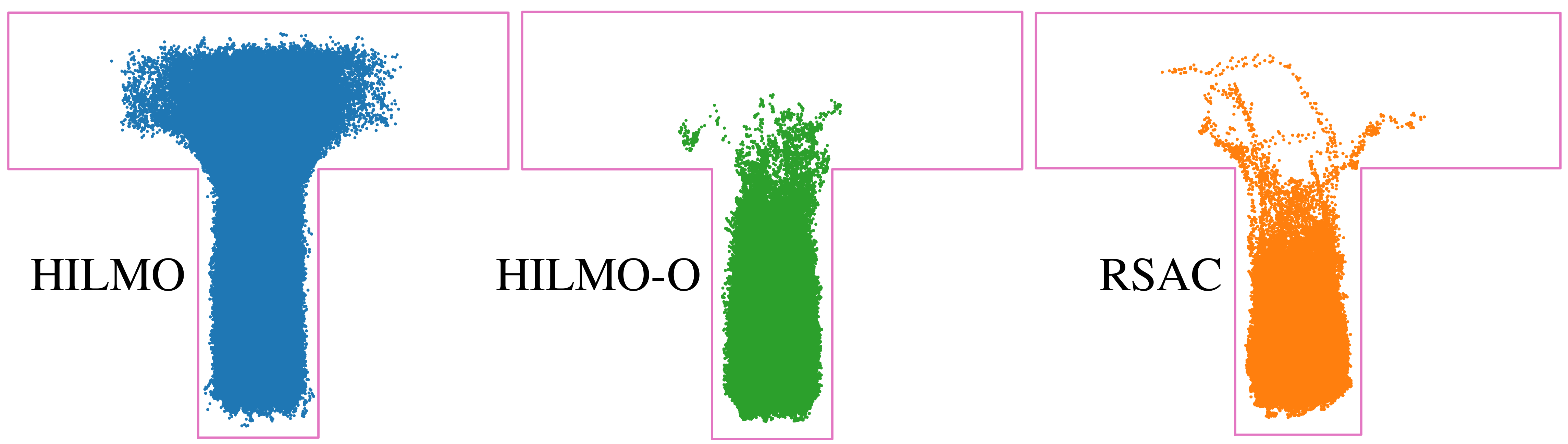}
    \caption{Coverage of HILMO, HILMO-O, and RSAC in \texttt{Ant-Heaven-Hell} after $500$k timesteps of training.}
    \label{fig:state_visitation_viz}
\end{wrapfigure}
\noindent \underline{Efficient Exploration.} We hypothesize that HILMO performs better because of a more effective exploration. To validate, we visualize the ant's positions in \texttt{Ant-Heaven-Hell} during 500k environment timesteps of training in Fig.~\ref{fig:state_visitation_viz}. Apparently, HILMO explores better thanks to high-level actions, covering the T-shaped space densely. HILMO-O also utilizes high-level actions, but its bottom policy rarely achieves given goals, which is detrimental to the final performance. The reason is that without penalizing unachieved goals, the top policy of HILMO-O is free to propose unrealistic goals (\emph{e.g.}, outside of the working space). This issue is also present in the original HIRO agents (see this \href{https://sites.google.com/view/efficient-hrl}{hyper-link}). In \texttt{Two-Boxes} and \texttt{Door-Push}, the bottom policy of HILMO-O performs better, therefore, its performance is relatively comparable to that of HILMO.

\noindent \underline{Other Baselines.} With no memory, SAC could not solve any domains as expected. Surprisingly, a memory-less HAC can perform quite well in \texttt{Ant-Tag} due to a well-trained bottom policy that sometimes can corner and tag the opponent. RSAC can only succeed in \texttt{Two-Boxes} for some seeds. It struggles to learn in the other tasks in which episodes last longer, and the reward sparsity will hinder learning more severely. DPFRL and VRM surprisingly perform poorly across domains. These methods have no mechanisms to deal with sparse rewards and were tested only on POMDPs that require no active information gatherings (\emph{e.g.}, flickering Atari games, or locomotion tasks with hidden velocities or positions).

\begin{table}[t]
\centering
\setlength{\tabcolsep}{1em}
\caption{Wall-clock training time in \textit{hours} for selected agents in Fig.~\ref{fig:all_success_rates}.}
\label{tab:all_hours}
\begin{tabular}{@{}lcccr@{}}\toprule
\textbf{Domain} & \textbf{RSAC} & \textbf{HILMO-O} & \textbf{HILMO} \\ \midrule
\texttt{Two-Boxes} & \multicolumn{1}{r}{3.31 $\pm$ 0.2 } & \multicolumn{1}{r}{\textbf{0.96} $\pm$ \textbf{0.2} } & \multicolumn{1}{r}{1.04 $\pm$ 0.3} \\
\texttt{Ant-Heaven-Hell} & \multicolumn{1}{r}{71.02 $\pm$ 0.3 } & \multicolumn{1}{r}{11.31 $\pm$ 0.2} & \multicolumn{1}{r}{\textbf{4.69} $\pm$ \textbf{0.3} }\\
\texttt{Ant-Tag} & \multicolumn{1}{r}{68.25 $\pm$ 0.4 } & \multicolumn{1}{r}{12.51 $\pm$ 0.3} & \multicolumn{1}{r}{\textbf{6.41} $\pm$ \textbf{0.3}}\\
\texttt{Door-Push} & \multicolumn{1}{r}{112.7 $\pm$ 0.4 } & \multicolumn{1}{r}{26.10 $\pm$ 0.2}& \multicolumn{1}{r}{\textbf{23.2} $\pm$ \textbf{0.2} }\\
\bottomrule
\end{tabular}
\end{table}

\vspace{5pt}
\noindent \textbf{Wall-clock Training Time. } To measure the time fairly, we train only one experiment at a time on the same CPU (Intel i7-8700K 3.7GHz with 12 processors) and GPU (Nvidia GeForce GTX 1080 8GB). Moreover, we excluded other baselines that did not learn or use GPU in their implementations. Table~\ref{tab:all_hours} shows that our agents (HILMO and HILMO-O) take significantly less time to train than RSAC in all domains.  While they are relatively comparable in \texttt{Two-Boxes}, HILMO-O is slower than HILMO in the remaining three domains. In these domains, while RSAC does not learn after days of training, HILMO can learn good policies in a reasonable time. The acceleration can be attributed to shorter top-level episodes for HILMO because each top-level observation summarizes several primitive observations. Moreover, the bottom policy trained in parallel can learn significantly faster due to full observability (see Appendix \ref{app:goal-ratios}).

\begin{figure}[t]
    \centering
    \includegraphics[width=0.9\linewidth]{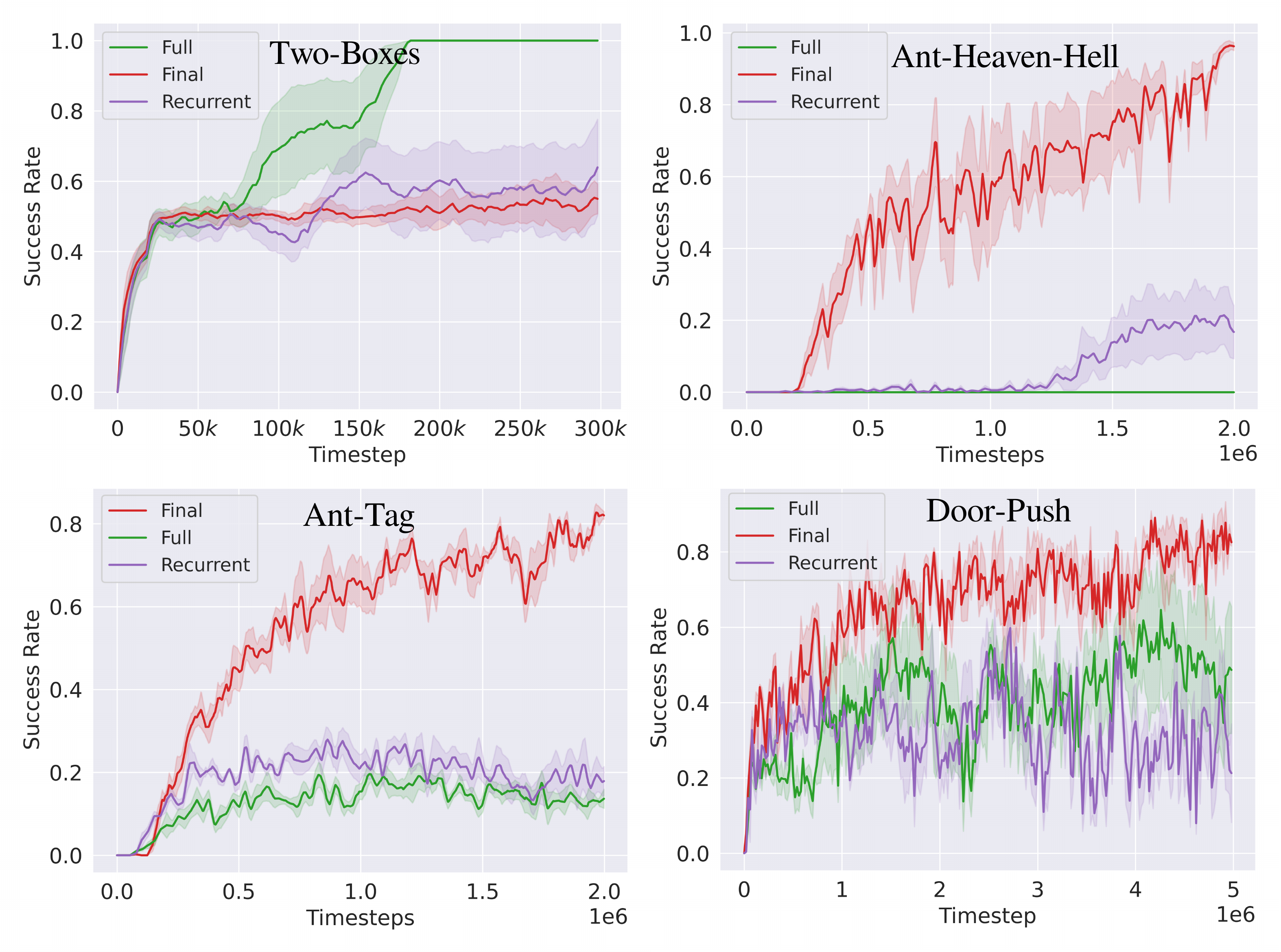}
    \caption{Comparing the success rates of different top-level observations (4 seeds).}
    \label{fig:summarizers}
    \vspace{-0.5cm}
\end{figure}

\subsection{Comparing Full, Final, and Recurrent}
The comparison is depicted in Fig.~\ref{fig:summarizers}. Generally, \specialword{Final} (red) is the best strategy, potentially due to more concise information. It dominates in three out of four domains, only be outperformed by \specialword{Full} (green) in \texttt{Two-Boxes}. \specialword{Final} does not perform well in \texttt{Two-Boxes} possibly because the proposed goals must be precisely on top of the two boxes for the final observation to contain angular changes. In contrast, other strategies are less restricted. \specialword{Recurrent} (purple) introduces another recurrent component into the agent, which seems to hinder learning.

\subsection{Robot Experiments for Two-Boxes}
The learned policy is deployed on a UR5e robot arm (Fig.~\ref{fig:policy_viz} left) with a specialized gripper~\cite{schwarm2019floating} that can control the compliance of its two fingers using hydrostatic actuators. However, we only use the left finger and keep it compliant at all times. We adjust the distance between the fingertip and the table so that deflected angles behave similarly in the simulation. We use wooden boxes of two sizes, making up four configurations, and perform three runs for each configuration with a random initial position of the finger between the two boxes. All runs are successful. Please refer to \url{https://youtu.be/KiVIBkdm0U8} for the demonstration of the policy as well as the policies learned in other domains.

\begin{wrapfigure}{R}{0.5\linewidth}
    \centering
    \includegraphics[width=1.0\linewidth]{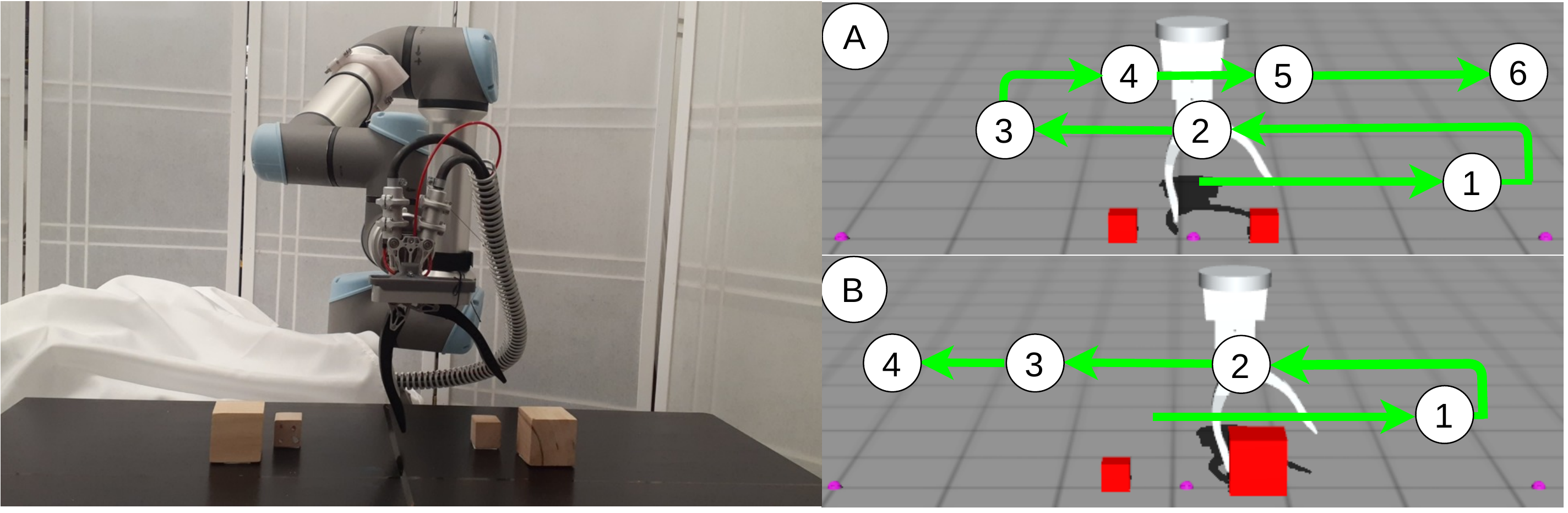}
    \caption{(Left) A UR5e robot and wooden boxes to deploy the learned policy. (Right) The policy learned in \texttt{Two-Boxes} with two distinct box configurations. White circles denote goals generated by the top-level policy. }
    \label{fig:policy_viz}
\end{wrapfigure}

\noindent \textbf{Policy Visualization.} The right side of Fig.~\ref{fig:policy_viz} illustrates the policy learned by our agent in \texttt{Two-Boxes} with two configurations of boxes: small-small and small-big. The learned policy generates both informative and rewarding actions based on the history of the left finger's angles and positions. The agent first goes right until passing a box; then, it backtracks until passing the other box. If the angle history indicates that the two boxes have the same size, the agent will go to the right end through goal 6 in case A to finish the task. Otherwise, it will go to the left end through goal 4 in case B. In \texttt{Ant-Heaven-Hell}, we also notice interesting patterns when visualizing the evolution of cell memories of a trained HILMO agent (see Appendix \ref{app:evolve}).

\section{Conclusion and Future Work}
\label{sect:conclusion}
This work introduces MOB-MOMDPs, a subclass of MOMDPs which can be found all across active robotic research areas, in which the agent's actions only have a direct effect on the fully observable component of the state.  We also introduce HILMO, a hierarchical agent that exploits the mixed observability assumptions of MOB-MOMDPs.  Our empirical evaluation shows that HILMO achieves improved learning performance and training time. A policy learned entirely in simulation is effectively deployed on real hardware. 

Even we focus on continuous control tasks, extending our approach for discrete control tasks is straightforward. Moreover, a hierarchical agent further allows state abstractions in which the input might be optimized for each level for more efficient learning. For instance, the top policy can propose optimal goals in \texttt{Ant-Tag} and \texttt{Ant-Heaven-H} \texttt{ell} without using the joints' angles and velocities.
\subsubsection{Acknowledgements}
This material is supported by the Army Research Office award W911NF20-1-0265 and the NSF grant 1816382.

%
%
%
%

\bibliographystyle{splncs04}
\bibliography{refs}

\begin{thebibliography}{10}
\providecommand{\url}[1]{\texttt{#1}}
\providecommand{\urlprefix}{URL }
\providecommand{\doi}[1]{https://doi.org/#1}

\bibitem{app_cho2014learning}
Cho, K., Van~Merri{\"e}nboer, B., Gulcehre, C., Bahdanau, D., Bougares, F.,
  Schwenk, H., Bengio, Y.: Learning phrase representations using rnn
  encoder-decoder for statistical machine translation. arXiv preprint
  arXiv:1406.1078  (2014)

\bibitem{app_fujimoto2018addressing}
Fujimoto, S., Hoof, H., Meger, D.: Addressing function approximation error in
  actor-critic methods. In: International Conference on Machine Learning. pp.
  1587--1596. PMLR (2018)

\bibitem{app_heess2015memory}
Heess, N., Hunt, J.J., Lillicrap, T.P., Silver, D.: Memory-based control with
  recurrent neural networks. arXiv preprint arXiv:1512.04455  (2015)

\bibitem{app_hochreiter1997long}
Hochreiter, S., Schmidhuber, J.: Long short-term memory. Neural computation
  \textbf{9}(8),  1735--1780 (1997)

\bibitem{app_kingma2014adam}
Kingma, D.P., Ba, J.: Adam: A method for stochastic optimization. arXiv
  preprint arXiv:1412.6980  (2014)

\bibitem{pytorchrl}
Kostrikov, I.: Pytorch implementations of reinforcement learning algorithms.
  \url{https://github.com/ikostrikov/pytorch-a2c-ppo-acktr-gail} (2018)

\bibitem{app_lillicrap2015continuous}
Lillicrap, T.P., Hunt, J.J., Pritzel, A., Heess, N., Erez, T., Tassa, Y.,
  Silver, D., Wierstra, D.: Continuous control with deep reinforcement
  learning. arXiv preprint arXiv:1509.02971  (2015)

\bibitem{app_HIRO}
Nachum, O., Gu, S.S., Lee, H., Levine, S.: Data-efficient hierarchical
  reinforcement learning. In: Advances in Neural Information Processing
  Systems. vol.~31 (2018)

\bibitem{app_schulman2017proximal}
Schulman, J., Wolski, F., Dhariwal, P., Radford, A., Klimov, O.: Proximal
  policy optimization algorithms. arXiv preprint arXiv:1707.06347  (2017)

\bibitem{app_yang2021recurrent}
Yang, Z., Nguyen, H.: Recurrent off-policy baselines for memory-based
  continuous control. arXiv preprint arXiv:2110.12628  (2021)

\end{thebibliography}


\begin{thebibliography}{10}
\providecommand{\url}[1]{\texttt{#1}}
\providecommand{\urlprefix}{URL }
\providecommand{\doi}[1]{https://doi.org/#1}

\bibitem{andrychowicz2017hindsight}
Andrychowicz, M., Wolski, F., Ray, A., Schneider, J., Fong, R., Welinder, P.,
  McGrew, B., Tobin, J., Abbeel, P., Zaremba, W.: Hindsight experience replay.
  arXiv preprint arXiv:1707.01495  (2017)

\bibitem{astrom1965optimal}
Astrom, K.J.: Optimal control of markov decision processes with incomplete
  state estimation. J. Math. Anal. Applic.  \textbf{10},  174--205 (1965)

\bibitem{chen2018planning}
Chen, M., Nikolaidis, S., Soh, H., Hsu, D., Srinivasa, S.: Planning with trust
  for human-robot collaboration. In: Proceedings of the 2018 ACM/IEEE
  International Conference on Human-Robot Interaction. pp. 307--315 (2018)

\bibitem{cho2014learning}
Cho, K., Van~Merri{\"e}nboer, B., Gulcehre, C., Bahdanau, D., Bougares, F.,
  Schwenk, H., Bengio, Y.: Learning phrase representations using rnn
  encoder-decoder for statistical machine translation. arXiv preprint
  arXiv:1406.1078  (2014)

\bibitem{chung2011search}
Chung, T.H., Hollinger, G.A., Isler, V.: Search and pursuit-evasion in mobile
  robotics. Autonomous robots  \textbf{31}(4),  299--316 (2011)

\bibitem{dietterich2000hierarchical}
Dietterich, T.G.: Hierarchical reinforcement learning with the maxq value
  function decomposition. Journal of artificial intelligence research
  \textbf{13},  227--303 (2000)

\bibitem{fujimoto2018addressing}
Fujimoto, S., Hoof, H., Meger, D.: Addressing function approximation error in
  actor-critic methods. In: International Conference on Machine Learning. pp.
  1587--1596. PMLR (2018)

\bibitem{haarnoja2018soft}
Haarnoja, T., Zhou, A., Abbeel, P., Levine, S.: Soft actor-critic: Off-policy
  maximum entropy deep reinforcement learning with a stochastic actor. In:
  International conference on machine learning. pp. 1861--1870. PMLR (2018)

\bibitem{han2019variational}
Han, D., Doya, K., Tani, J.: Variational recurrent models for solving partially
  observable control tasks. In: 8th International Conference on Learning
  Representations, {ICLR} (2020)

\bibitem{heess2015memory}
Heess, N., Hunt, J.J., Lillicrap, T.P., Silver, D.: Memory-based control with
  recurrent neural networks. arXiv preprint arXiv:1512.04455  (2015)

\bibitem{hernandez2001hierarchical}
Hernandez-Gardiol, N., Mahadevan, S.: Hierarchical memory-based reinforcement
  learning. Advances in Neural Information Processing Systems pp. 1047--1053
  (2001)

\bibitem{hochreiter1997long}
Hochreiter, S., Schmidhuber, J.: Long short-term memory. Neural computation
  \textbf{9}(8),  1735--1780 (1997)

\bibitem{hsiao2007grasping}
Hsiao, K., Kaelbling, L.P., Lozano-Perez, T.: Grasping pomdps. In: Proceedings
  2007 IEEE International Conference on Robotics and Automation. pp.
  4685--4692. IEEE (2007)

\bibitem{kulkarni2016hierarchical}
Kulkarni, T.D., Narasimhan, K., Saeedi, A., Tenenbaum, J.: Hierarchical deep
  reinforcement learning: Integrating temporal abstraction and intrinsic
  motivation. Advances in neural information processing systems  \textbf{29},
  3675--3683 (2016)

\bibitem{le2018deep}
Le, T.P., Vien, N.A., Chung, T.: A deep hierarchical reinforcement learning
  algorithm in partially observable markov decision processes. Ieee Access
  \textbf{6},  49089--49102 (2018)

\bibitem{levy2017learning}
Levy, A., Konidaris, G.D., Jr., R.P., Saenko, K.: Learning multi-level
  hierarchies with hindsight. In: 7th International Conference on Learning
  Representations, {ICLR} (2019)

\bibitem{HIPPO}
Li, A.C., Florensa, C., Clavera, I., Abbeel, P.: Sub-policy adaptation for
  hierarchical reinforcement learning. In: 8th International Conference on
  Learning Representations, {ICLR} (2020)

\bibitem{lillicrap2015continuous}
Lillicrap, T.P., Hunt, J.J., Pritzel, A., Heess, N., Erez, T., Tassa, Y.,
  Silver, D., Wierstra, D.: Continuous control with deep reinforcement
  learning. arXiv preprint arXiv:1509.02971  (2015)

\bibitem{Ma2020Discriminative}
Ma, X., Karkus, P., Hsu, D., Lee, W.S., Ye, N.: Discriminative particle filter
  reinforcement learning for complex partial observations. In: 8th
  International Conference on Learning Representations, {ICLR} (2020)

\bibitem{HIRO}
Nachum, O., Gu, S.S., Lee, H., Levine, S.: Data-efficient hierarchical
  reinforcement learning. In: Advances in Neural Information Processing
  Systems. vol.~31 (2018)

\bibitem{nikolaidis2017human}
Nikolaidis, S., Zhu, Y.X., Hsu, D., Srinivasa, S.: Human-robot mutual
  adaptation in shared autonomy. In: 2017 12th ACM/IEEE International
  Conference on Human-Robot Interaction (HRI. pp. 294--302. IEEE (2017)

\bibitem{ong2009pomdps}
Ong, S.C., Png, S.W., Hsu, D., Lee, W.S.: Pomdps for robotic tasks with mixed
  observability. In: Robotics: Science and Systems. vol.~5, p.~4 (2009)

\bibitem{precup1997multi}
Precup, D., Sutton, R.S.: Multi-time models for temporally abstract planning.
  Advances in neural information processing systems  \textbf{10} (1997)

\bibitem{schwarm2019floating}
Schwarm, E., Gravesmill, K.M., Whitney, J.P.: A floating-piston hydrostatic
  linear actuator and remote-direct-drive 2-dof gripper. In: 2019 international
  conference on robotics and automation (ICRA). pp. 7562--7568. IEEE (2019)

\bibitem{steckelmacher2018reinforcement}
Steckelmacher, D., Roijers, D.M., Harutyunyan, A., Vrancx, P., Plisnier, H.,
  Now{\'e}, A.: Reinforcement learning in pomdps with memoryless options and
  option-observation initiation sets. In: Thirty-second AAAI conference on
  artificial intelligence (2018)

\bibitem{todorov2012mujoco}
Todorov, E., Erez, T., Tassa, Y.: Mujoco: A physics engine for model-based
  control. In: 2012 IEEE/RSJ International Conference on Intelligent Robots and
  Systems. pp. 5026--5033. IEEE (2012)

\bibitem{wang2016impact}
Wang, N., Pynadath, D.V., Hill, S.G.: The impact of pomdp-generated
  explanations on trust and performance in human-robot teams. In: AAMAS. pp.
  997--1005 (2016)

\bibitem{wiering1997hq}
Wiering, M., Schmidhuber, J.: Hq-learning. Adaptive Behavior  \textbf{6}(2),
  219--246 (1997)

\bibitem{xiao2019online}
Xiao, Y., Katt, S., ten Pas, A., Chen, S., Amato, C.: Online planning for
  target object search in clutter under partial observability. In: 2019
  International Conference on Robotics and Automation (ICRA). pp. 8241--8247.
  IEEE (2019)

\bibitem{yang2021recurrent}
Yang, Z., Nguyen, H.: Recurrent off-policy baselines for memory-based
  continuous control. Deep RL Workshop NeurIPS  (2021)

\end{thebibliography}


\clearpage
\setcounter{page}{1}
\setcounter{figure}{0}
\appendix
\section{HIerarchical reinforcement Learning under Mixed Observability (HILMO)}
\label{app:HILMO}
\begin{algorithm}[h!]
  \caption{HILMO}
  \label{alg:algo}
  \begin{algorithmic}[1]
    \State{\textbf{Constants:} Horizons: $H$ (whole agent), $H^T$ (top-level), $H^\mathcal{X}$ (bottom-level, \emph{i.e.}, $k$), goal testing probability: $\lambda$}
    \State {Replay buffers: $\mathcal{D}^T$, $\mathcal{D}^\mathcal{X}$; actors and critics $(\pi^T, Q^T), (\pi^\mathcal{X}, Q^\mathcal{X})$}
    \For {as many episodes}
    \State {Run-Top($\lambda$)}
    \State{Update $\pi^T, Q^T$ using RDPG \citeApp{app_heess2015memory} from episodes in $\mathcal{D}^T$}
    \State{Update $\pi^\mathcal{X}, Q^\mathcal{X}$ using DDPG \citeApp{app_lillicrap2015continuous} from transitions in $\mathcal{D}^\mathcal{X}$}
    \EndFor
    
    \Function{Run-Top}{\hspace{0.1mm}}
    \State{Empty storage $\tau \leftarrow \emptyset$, history $h^T \leftarrow h^T_{\text{init}}$, top observation $o^T \leftarrow  o^T_{\text{init}}$}
    \For{$H$ steps or until the environment solved}
    \State{$x^g \leftarrow \pi^T(h^T) + \text{exploration noise}$ }
    \Comment{Sample a goal for the bottom policy}
    \State{test-goal $\leftarrow$ True w.p. $\lambda$}
    \State{$[o^T, x^g, o'^{T}, r]$ = Run-Bottom($x^g$, test-goal)}
    \If{$x^g$ is tested and missed}
    \State{$\tau \leftarrow \tau \cup [o^T, x^g, o'^{T}, r=-H^T]$}     \Comment{Store a subgoal testing transition}
    \State{$\tau \leftarrow \tau \cup [o^T, x_{\text{met}}^{g}, o'^{T}, r]$} \Comment{Store a hindsight action transition}
    \Else
    \State{$\tau \leftarrow \tau \cup [o^T, x^g, o'^{T}, r]$}
    \EndIf
    \State{$h^T \leftarrow h^T \cup o'^{T}$}
    \Comment{Use recurrent module}
    \State{$o^T \leftarrow o'^{T}$}
    \EndFor
    \State{Process transitions in $\tau$ (Section \ref{ref:top-learn}) to create episodes and store in $\mathcal{D}^T$}
    \EndFunction
    
    \Function{Run-Bottom}{$x^g$, test-goal}
    \State {$x, z \leftarrow o$; $g \leftarrow x^g$}
    \Comment{Set observation and goal for the bottom level}
    \State {$E \leftarrow \emptyset$}
    \Comment{Empty storage for hindsight goal transitions (HGT)}
    \For{$H^\mathcal{X}$ attempts or until $g$ achieved}
    \State{$a \leftarrow \pi^\mathcal{X}(x, g)$ + exploration noise (if not test-goal)}
    \State{Execute $a$ in environment, observe $o'=(x',z')$ and environment reward $r$}
    \State{$\mathcal{D}^\mathcal{X} \leftarrow [x, a, x', r^\mathcal{X} \in \{ -1, 0\}, g]$}
    \Comment{Store transitions with intrinsic rewards}
    \State{$E \leftarrow [x, a, x', r^\mathcal{X}=\emptyset, g=\emptyset]$}
    \Comment{Store HGTs with empty rewards \& goals}
    \State{$x \leftarrow x'$}
    \EndFor
    \State{$\mathcal{D}^\mathcal{X} \leftarrow $ Complete HGTs}
    \Comment{Determine goals and rewards for HGTs in $E$}
    \State{Create next top-level observation $o'^{T}$ from all $o$-s using Full, Final, or Recurrent}
    \State{\textbf{return} $[o^T, x^g, o'^{T}, \sum r]$}
    \Comment{Return a top transition ($o^T$ from Run-Top)}
    \EndFunction
  \end{algorithmic}
\end{algorithm}
\noindent \underline{Parameters:}
\begin{itemize}
    \item $\lambda = 0.3$
    \item $[H, H^\mathcal{X}]$: [400, 20] (for \texttt{Ant-Tag}, \texttt{Ant-Heaven-Hell}, and \texttt{Door-Push}), [100, 12] (\texttt{Two-Boxes})
\end{itemize}

\clearpage
\section{Deep Deterministic Policy Gradient (DDPG)}
\label{app:ddpg}
\begin{algorithm}[htbp]
  \caption{DDPG algorithm (given a replay buffer of transitions and without target networks)}
  \begin{algorithmic}[1]
    \State {Initialize critic network $Q(s,a \mid  \theta^Q)$ and actor $\mu(s\mid \theta^\mu)$ with weights $\theta^Q$ and $\theta^\mu$}
    \State{Given a replay buffer $R$}
    \For {$M$ episodes}
    \State{Sample a minibatch of $N$ transitions $(s_i, a_i, r_i, s_{i+1})$ from $R$}
    \State{Calculate target $$y_i = r_i + \gamma Q\left(s_{i+1}, \mu(s_{i+1}\mid \theta^\mu)\mid \theta^Q \right)$$}
    \State {Update critic by minimizing the loss $$\mathcal{L} \approx \frac{1}{N} \sum_i \left( y_i - Q(s_i, a_i \mid  \theta^Q) \right)^2$$}
    \State{Update actor using sampled policy gradient:
    $$ \Delta_{\theta^\mu} J \approx \frac{1}{N} \Delta_a Q(s,a \mid \theta^Q) \big| _{s=s_i, a=\mu(s_i)} \Delta_{\theta^\mu} \mu (s \mid \theta^{\mu}) |_{s_i}$$}
    \EndFor
    
  \end{algorithmic}
\end{algorithm}
\noindent \underline{Implementation details:}
\begin{itemize}
    \item Actor network architecture: (FC-64 + ReLU) + (FC-64 + ReLU) + (FC-action-dim + Tanh)
    \item Critic network architecture:  (FC-64 + ReLU) + (FC-64 + ReLU) + (FC-1)
    \item Replay buffer: 100k transitions
    \item Batch size: 1024
    \item Optimizer: Adam~\citeApp{app_kingma2014adam} with a learning rate of 0.001 and other default parameters
\end{itemize}

\clearpage
\section{Recurrent Deterministic Policy Gradient (RDPG)}
\label{app:rdpg}
\begin{algorithm}[h!]
  \caption{RDPG algorithm (given a replay buffer of episodes)}
  \begin{algorithmic}[1]
    \State {Initialize critic network $Q^\omega(a_t, h_t)$ and actor $\pi^\theta(h_t)$ with parameters $\omega$ and $\theta$}
    \State {Initialize target networks $Q^{\omega'}$ and actor $\pi^{\theta'}$ with parameters $\omega' \leftarrow \omega $ and $\theta' \leftarrow \theta$}
    \State{Given a replay buffer $R$ of episodes}
    \For {$M$ episodes}
    \State{Sample a minibatch of $N$ episodes from $R$}
    \State{Construct histories $h_t^i = (o_1^i, a_1^i, \dots, a_{t-1}^i, o_t^{i})$}
    \State {Compute target values for each sample episode $(y_1^i, \dots, y_T^i)$ using the recurrent target networks
    $$y_t^i = r_t^i + \gamma Q^{\omega'}(h_{t+1}^i, \pi^{\theta'}(h_{t+1}^i))$$}
    \State {Update critic (using back-propagation through time)
    $$\Delta \omega = \frac{1}{NT}\sum_i \sum_t \left(y_t^i - Q^\omega(h_t^i, a_t^i)\right) \frac{\partial Q^\omega(h_i^t, a_t^i)}{\partial \theta}$$}
    \State{Update actor
    $$\Delta \theta = \frac{1}{NT}\sum_i \sum_t \frac{\partial Q^\omega \left(h_t^i, \pi^\theta(h_t^i)\right) }{\partial a} \frac{\partial \pi^\theta(h_i^t)}{\partial \theta}$$}
    \State{Update actor and critic using Adam \citeApp{app_kingma2014adam}}
    \State{Update target networks
    \begin{align*}
        \omega' &\leftarrow \tau \omega + (1-\tau)\omega' \\
        \theta' &\leftarrow \tau \theta + (1-\tau)\theta'
    \end{align*}
    }
    \EndFor
  \end{algorithmic}
\end{algorithm}
\noindent \underline{Implementation details:}
\begin{itemize}
    \item Actor network architecture: (LSTM-64) + (FC-64 + ReLU) + (FC-action-dim + Tanh)
    \item Critic network architecture: (LSTM-64) + (FC-64 + ReLU) + (FC-1)
    \item Replay buffer: from 5k-10k episodes
    \item Batch size: 256
    \item Optimizer: Adam~\citeApp{app_kingma2014adam} with a learning rate of 3e-4 and other default parameters
\end{itemize}

\clearpage
\section{Different Design Choices}
\label{app:design_choice}
\noindent \textbf{GRU v.s. LSTM.} We compare the performance of HILMO in \texttt{Ant-Tag} and \texttt{Two-Boxes} when using GRU \citeApp{app_cho2014learning} and LSTM \citeApp{app_hochreiter1997long} in Fig.~\ref{fig:gru_lstm}. Using GRU results in a slower speed of learning in \texttt{Ant-Tag}, but there is no major difference in \texttt{Two-Boxes}.

\begin{figure}[htbp]
    \centering
    \includegraphics[width=0.85\linewidth]{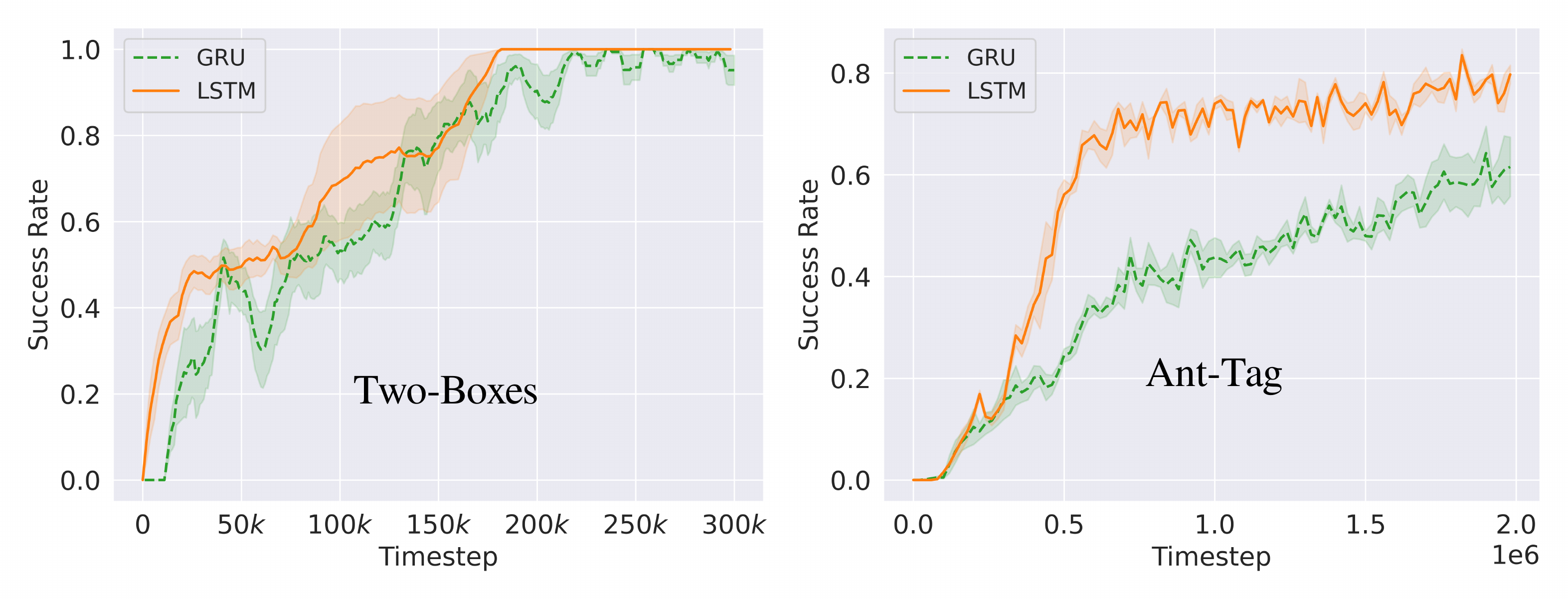}
    \caption{Performance comparison of HILMO when using LSTM and GRU for the top policy (4 seeds).}
    \label{fig:gru_lstm}
\end{figure}
\noindent \textbf{RDPG v.s. RTD3 for Top Policy.}
We compare the performance when using RDPG and RTD3 to learn the top-level policy in \texttt{Ant-Tag} and \texttt{Two-Boxes} in Fig.~\ref{fig:rtd3_rdpg}. RDPG outperformed RTD3 in the two domains, therefore we chose RDPG which is simpler to implement.
\begin{figure}[htbp]
    \centering
    \includegraphics[width=0.85\linewidth]{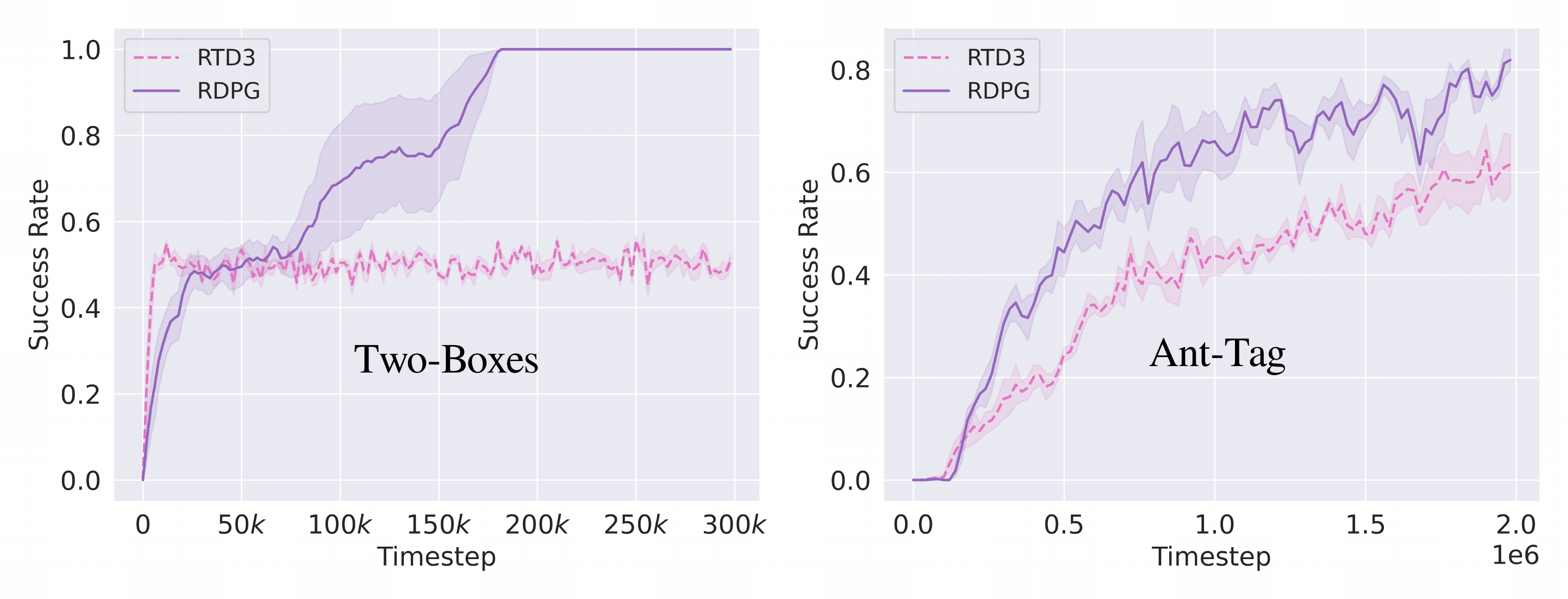}
    \caption{Using RTD3 and RDPG to learn the top policy (4 seeds).}
    \label{fig:rtd3_rdpg}
\end{figure}

\clearpage
\section{Stochastic MOB-MOMDPs and HILMO Optimality}
\label{app:stochasticity}
Theorem \ref{thm:optimality} assumes \emph{deterministic} MOB-MOMDPs, \emph{i.e.}, that $T^\mathcal{X}$ is deterministic, while $T^\mathcal{Y}$ is free to be stochastic (see Section~\ref{sect:mob-momdp}).  Here we provide a highly stochastic counterexample that shows why such deterministic transitions are necessary to guarantee the optimality of the HILMO approach.  Finally, we will argue that small levels of stochasticity (as might be found in more realistic navigation tasks) do not constitute a significant concern and should still imply that the HILMO approach is quasi-optimal.

\definecolor{color1}{RGB}{100,10,10}
\definecolor{color2}{RGB}{10,100,10}
\definecolor{color3}{RGB}{10,10,100}

\begin{figure}
\centering
\begin{subfigure}{0.45\textwidth}
\begin{tikzpicture}[
scale=0.9,
xnode/.style={circle, draw=gray, fill=gray!5, very thick, minimum size=7},
anode/.style={minimum size=0, inner sep=0mm},
edge/.style={->, very thick},
]

\node[xnode] (x) at (0, 0) {$x$};

\node[anode] (a1) at (1.5, 1.5) {$a_1$};
\node[anode] (a2) at (1.5, 0) {$a_2$};
\node[anode] (a3) at (1.5, -1.5) {$a_3$};

\node[xnode] (x1) at (4, 1.5) {$x'_1$};
\node[xnode] (x2) at (4, 0) {$x'_2$};
\node[xnode] (x3) at (4, -1.5) {$x'_3$};

\node (r1) at (5, 1.5) {$r=1$};
\node (r2) at (5, 0) {$r=1$};
\node (r3) at (5.1, -1.5){$r=-1$};

\draw[edge,color1] (x) to[out=20,in=180] (a1) to[out=0,in=180] node[right,above,pos=0.1] {50\%} (x1);
\draw[edge,color1]                       (a1) to[out=340,in=160] node[right,below,pos=0.1] {50\%} (x2);

\draw[edge,color2] (x) to[out=0,in=180] (a2) to[out=0,in=180] node[above,pos=0.1] {50\%} (x2);
\draw[edge,color2]                      (a2) to[out=340,in=160] node[below,pos=0.1] {50\%} (x3);

\draw[edge,color3] (x) to[out=340,in=180] (a3) to[out=0,in=180] node[below,pos=0.1] {50\%} (x3);
\draw[edge,color3]                        (a3) to[out=20,in=200] node[above,pos=0.05] {50\%} (x1);

\end{tikzpicture}
\caption{High stochasticity.}\label{fig:example:high_stochasticity}
\end{subfigure}
\begin{subfigure}{0.45\textwidth}
\begin{tikzpicture}[
scale=0.9,
xnode/.style={circle, draw=gray, fill=gray!5, very thick, minimum size=7},
anode/.style={minimum size=0, inner sep=0mm},
edge/.style={->, very thick},
]

\node[xnode] (x) at (0, 0) {$x$};

\node[anode] (a1) at (1.5, 1.5) {$a_1$};
\node[anode] (a2) at (1.5, 0) {$a_2$};
\node[anode] (a3) at (1.5, -1.5) {$a_3$};

\node[xnode] (x1) at (4, 1.5) {$x'_1$};
\node[xnode] (x2) at (4, 0) {$x'_2$};
\node[xnode] (x3) at (4, -1.5) {$x'_3$};

\node (r1) at (5, 1.5) {$r=1$};
\node (r2) at (5, 0) {$r=1$};
\node (r3) at (5.1, -1.5){$r=-1$};

\draw[edge,color1] (x) to[out=20,in=180] (a1) to[out=0,in=180] node[right,above,pos=0.1] {95\%} (x1);
\draw[edge,color1]                       (a1) to[out=340,in=160] node[right,below,pos=0.1] {\phantom{0}5\%} (x2);

\draw[edge,color2] (x) to[out=0,in=180] (a2) to[out=0,in=180] node[above,pos=0.1] {95\%} (x2);
\draw[edge,color2]                      (a2) to[out=340,in=160] node[below,pos=0.1] {\phantom{0}5\%} (x3);

\draw[edge,color3] (x) to[out=340,in=180] (a3) to[out=0,in=180] node[below,pos=0.1] {95\%} (x3);
\draw[edge,color3]                        (a3) to[out=20,in=200] node[above,pos=0.05] {\phantom{0}5\%} (x1);

\end{tikzpicture}
\caption{Low stochasticity.}\label{fig:example:low_stochasticity}
\end{subfigure}
\caption{Examples of stochastic transitions in $\mathcal{X}$-space.}
\end{figure}
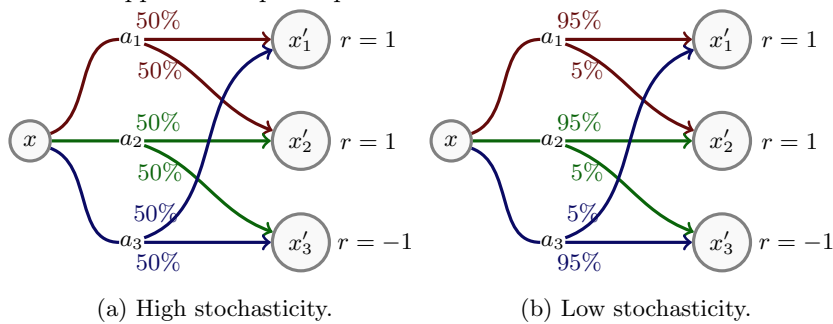

\noindent \textbf{High Stochasticity. } Consider the local dynamics depicted in Fig.~\ref{fig:example:high_stochasticity}, and assume that all other rewards are $0$, such that the optimality of a policy is exclusively determined by the behavior at $x$. In such a situation, it is desirable to reach either $x'_1$ or $x'_2$ and to avoid $x'_3$.  Given the stochastic dynamics in the example, this means that action $a_1$ is optimal, while $a_2$ and $a_3$ are suboptimal.  Note that $a_1$ does not guarantee to reach any specific $x'$;  it just guarantees that $x'_3$ is avoided. In the HILMO approach, the closest possible behavior would be to choose either $x'_1$ or $x'_2$ as the next ``pose''. However, if the top-level policy selects that it wants to reach $x'_1$, the issue arises that $x'_1$ can be reached by either $a_1$ or $a_3$, and neither action is better than the other for the goal of reaching $x'_1$. Therefore, the low-level policy is unable to prefer $a_1$ compared to $a_3$.  Similarly, if the top-level policy selects that it wants to reach $x'_2$, the issue arises that $x'_1$ can be reached by either $a_1$ or $a_2$, and neither action is better than the other for the goal of reaching $x'_1$. Therefore, the low-level policy is unable to prefer $a_1$ compared to $a_2$.  
Note that the optimal behavior of choosing $a_1$ to reach $x'_1$ or $x'_2$ \emph{can} be represented as a HILMO policy, it just is not intrinsically preferred to some other suboptimal behaviors.  In relation to Theorem~\ref{thm:optimality}, the issue is that under such stochasticity, the HILMO criteria for the low-level policy considers optimal and suboptimal behaviors to be equally valid.

This issue is potentially resolved by extending the HILMO approach to let the high-level policy select a whole subset of possible goals (or a smooth preference over goals) rather than a single goal, and let the low-level policy select actions which satisfy the aggregate goal as much as possible.  Such an extension is left for future work.  Instead, we continue the analysis, focusing on a low stochasticity example to show that the HILMO is able to handle low levels of stochasticity.

\noindent \textbf{Low Stochasticity. } Consider now the local dynamics depicted in Fig.~\ref{fig:example:low_stochasticity}, \emph{i.e.}, the same scenario described above, except that the transition probabilities associated with the actions have changed to be less stochastic.  In this case, action $a_1$, $a_2$, and $a_3$ are respectively more likely (albeit not guaranteed) to cause transitions to $x'_1$, $x'_2$, and $x'_3$.  A real-world analogy for such a situation could be a robot moving imperfectly through an environment due to minor wheel slippage.  In such a situation, if the top-level policy selects that it wants to reach $x'_1$, then the low-level policy has a concrete reason to prefer $a_1$ over $a_3$, i.e., that $a_1$ is more likely to transition to $x'_1$.  On the other hand, if the top-level policy selects that it wants to reach $x'_2$, then the low-level policy has a concrete reason to prefer $a_2$ over $a_1$, meaning that there is still a chance of reaching $x'_3$.  However, the top-level policy can take the low-level behavior into account and learn that choosing $x'_1$ as a target is preferred to $x'_2$ because choosing $x'_1$ as target never leads to negative rewards.  Therefore, optimal behavior is still guaranteed to be valued better than non-optimal behavior by the HILMO criteria.

We argue that, in practice, realistic motion-based tasks such as the ones modeled by MOB-MOMDPs tend to have low motion stochasticity rather than high motion stochasticity.  It follows that the HILMO approach should still be able to achieve quasi-optimal or even optimal behavior, even with low levels of stochasticity.

\clearpage
\section{HILMO with Off-Policy Corrections (HILMO-O)}
\label{app:HILMO-O}
\begin{algorithm}[h!]
  \caption{HILMO-O}
  \begin{algorithmic}[1]
    \State{\textbf{Constants:} Horizons: $H$ (whole agent), $H^\mathcal{X}$ (bottom, \emph{i.e.}, $k$)}
    \State {Replay buffers: $\mathcal{D}^T$, $\mathcal{D}^\mathcal{X}$; actors and critics $(\pi^T, Q^T), (\pi^\mathcal{X}, Q^\mathcal{X})$}
    \For {as many episodes}
    \State {Run-Top()}
    \State{Update $\pi^T, Q^T$ using RDPG from Off-Policy-Correct(episodes in $\mathcal{D}^T$, $\pi^\mathcal{X}$)}
    \State{Update $\pi^\mathcal{X}, Q^\mathcal{X}$ using DDPG from transitions in $\mathcal{D}^\mathcal{X}$}
    \EndFor
    
    \Function{Run-Top}{\hspace{0.1mm}}
    \State{Empty storage $\tau \leftarrow \emptyset$, history $h^T \leftarrow h^T_{\textrm{init}}$, top observation $o^T \leftarrow  o^T_{\textrm{init}}$}
    \For{$H$ steps or until the environment solved}
    \State{$x^g \leftarrow \pi^T(h^T) + \textrm{exploration noise}$}    \Comment{Sample a goal for the bottom policy}
    \State{$[o^T, x^g, o'^{T}, r], \{x_{1:c}, a_{1:c}, g_{1:c}\}$ = Run-Bottom($x^g$)}
    
    \State{$\tau \leftarrow \tau \cup [o^T, x^g, o'^{T}, r$, off-policy correction info = $\{x_{1:c}, a_{1:c}, g_{1:c}\}]$}
    
    \State{$h^T \leftarrow h^T \cup o'^{T}$}
    \Comment{Use recurrent module}
    \State{$o^T \leftarrow o'^{T}$}
    \EndFor
    \State{$\mathcal{D}^T \leftarrow \mathcal{D}^T \cup \tau$}
    \EndFunction
    
    \Function{Run-Bottom}{$x^g$}
    \State {$x, z \leftarrow o$; $g \leftarrow x^g$; $c \leftarrow 0$}

    \For{$H^\mathcal{X}$ attempts or until $g$ achieved}
    \State{$a \leftarrow \pi^\mathcal{X}(x, g) + \textrm{ exploration noise}$ }
    \Comment{Sample a noisy action}
    \State{Execute $a$ in environment, observe $o'=(x',z')$ and environment reward $r$}
    \State{$g' \leftarrow x + g - x'$}
    \Comment{Update relative goal}
    \State{$\mathcal{D}^\mathcal{X} \leftarrow [x, g, a, r^\mathcal{X} = - \|g'\|_2, x', g']$}
    \Comment{Store transitions with intrinsic rewards}
    \State{$x, g \leftarrow x', g'$}
    \State{c += 1}
    \EndFor
    \State{Create next top-level observation $o'^T$ from all $o$-s using Full, Final, or Recurrent}
    \State{\textbf{return} [$o^T$, $x^g$, $o'^{T}$, $\sum r$], $\{x_{1:c}, a_{1:c}, g_{1:c}\}$}
    \EndFunction
    
    \Function{Off-Policy-Correct}{$\{\tau_i\}$,$\pi^{\mathcal{X}}$}
    \Comment{For top-level}
    \For{each transition in $\{\tau_i\}$}
    \State{$\tilde{x}^g = \max_{g_1} \prod_{t=1}^c \pi^{\mathcal{X}} (a_{t} \mid x_{t}, g_{t})$}
     \Comment{Find a goal that makes $\pi^{\mathcal{X}}$ take the same actions as its past version}
    \State{Replace transition with $[o^T, \tilde{x}^g, o'^{T}, r]$}
    \EndFor
     \State{\textbf{return} $\{\tau_i\}$}
    \EndFunction
  \end{algorithmic}
\end{algorithm}
\noindent \underline{Parameters:}
\begin{itemize}
    \item Number of goal candidates used in maximization: 10 (as in HIRO~\citeApp{app_HIRO})
    \item Bottom reward scale: 1. Top reward scale: 1 (0.1 was used for HIRO; due to a different reward function)
    \item $[H, H^\mathcal{X}]$: [400, 20] (for \texttt{Ant-Tag}, \texttt{Ant-Heaven-Hell}, and \texttt{Door-Push}), [100, 12] (\texttt{Two-Boxes})
\end{itemize}

\clearpage
\section{Additional Baselines}
\label{app:additional_baselines}
We run additional experiments to measure the success rates (see Fig.~\ref{fig:additional-baselines}) of the following additional baselines:
\begin{itemize}
    \item Recurrent Proximal Policy Optimization (\underline{RPPO}) is the recurrent version of PPO \citeApp{app_schulman2017proximal} implemented in \citeApp{pytorchrl}
    \item Recurrent Twin Delayed Deep Deterministic (\underline{RTD3}) is the recurrent version of TD3 \citeApp{app_fujimoto2018addressing} from \citeApp{app_yang2021recurrent}
    \item Recurrent Deterministic Policy Gradient (\underline{RDPG}) \citeApp{app_heess2015memory} implemented in \citeApp{app_yang2021recurrent}
\end{itemize}
\begin{figure}[htbp]
    \centering
    \includegraphics[width=0.9\linewidth]{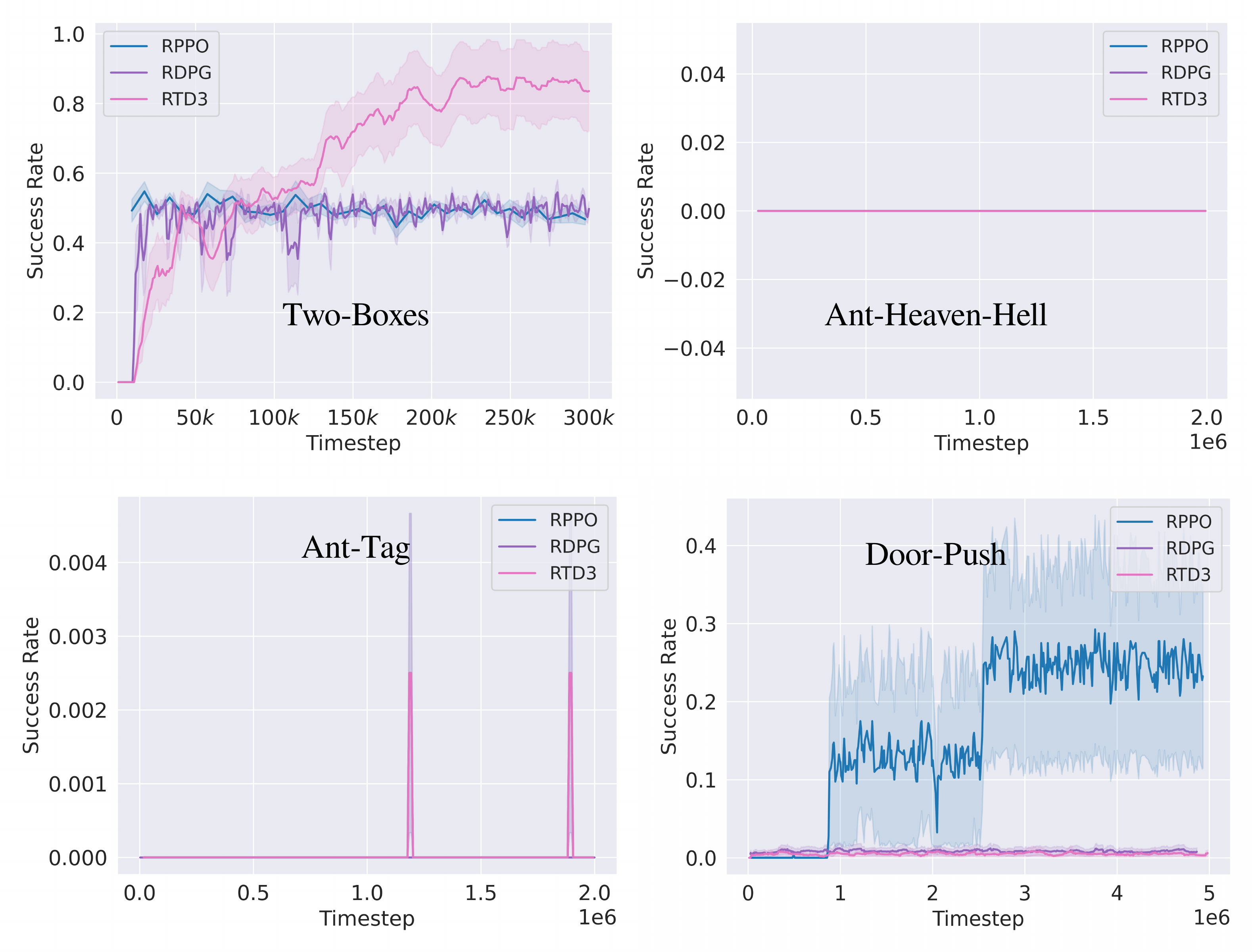}
    \caption{Success rates of additional baselines (4 seeds).}
    \label{fig:additional-baselines}
\end{figure}


\clearpage
\section{Goal Achieved Ratios of Trained Agents}
\label{app:goal-ratios}
Fig. \ref{fig:goal} shows the goal-achieved ratios and the success rates in \texttt{Two-Boxes} and \texttt{Ant-Heaven-Hell} of trained agents. As expected, the bottom policy with full observability learns quicker than the top policy with the goal-achieved ratio quickly climbing up, indicating the benefit of our approach. More interestingly, the bottom-level policy (red) does not have to be perfect for the whole agent to solve the given tasks satisfactorily. By inspecting the learned policies, we observe that the top-level policy will sometimes propose goals closer to the agent when it realizes that the last goal has not been achieved.
\begin{figure}[thb]
    \centering
    \includegraphics[width=0.85\linewidth]{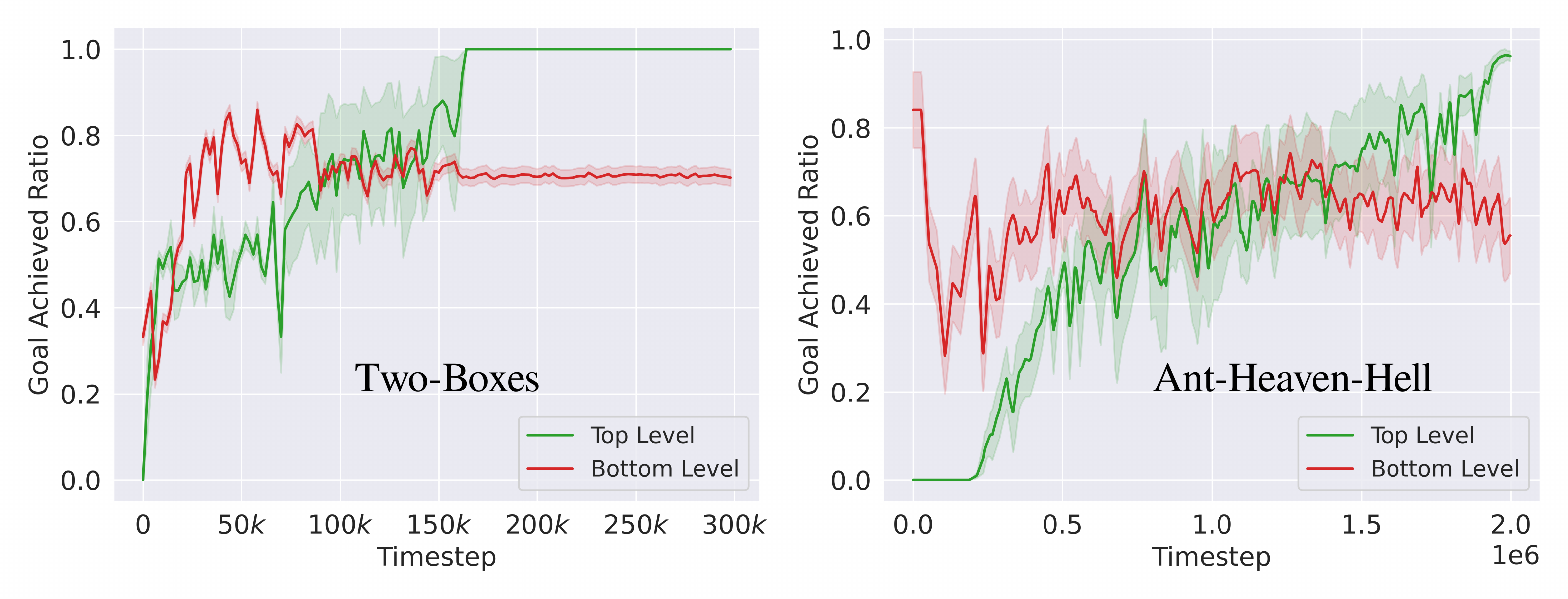}
    \caption{Goal-achieved ratios at two levels of fully trained agents (4 seeds).}
    \label{fig:goal}
\end{figure}

\section{Evolution of Memory Cells of a Trained Agent}
\label{app:evolve}
Fig.~\ref{fig:cell_evolve} shows the evolution of several memory cells of the LSTM inside the top-level actor in \texttt{Ant-Heaven-Hell} after training. We visualize over four episodes in which heaven is on the left (the left figure) and the right (the right figure). Apparently, memory cells behave differently depending on the side observed when the agent is inside the blue area (marked by shaded areas). The same behaviors generally repeat whenever the same side is observed.
\begin{figure}[thb]
    \centering
    \includegraphics[width=0.85\linewidth]{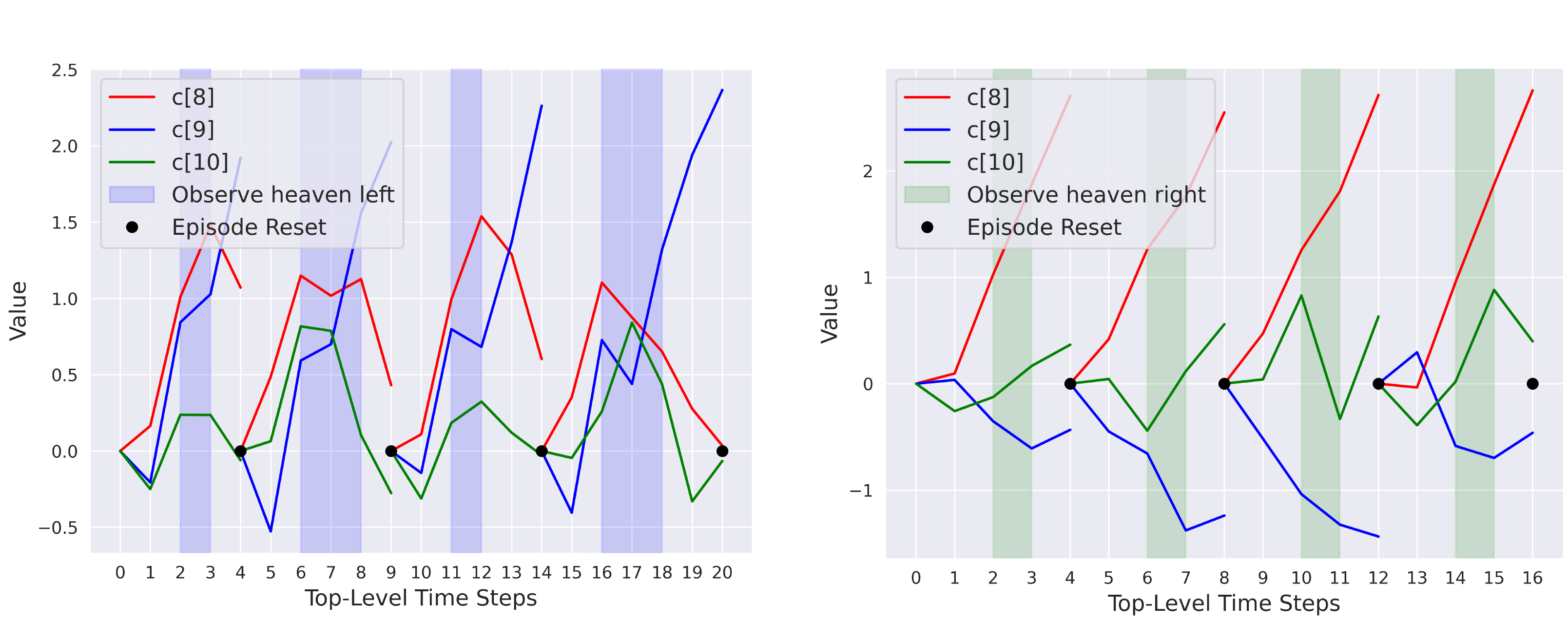}
    \caption{The memory cell's internal states $c[j]$ of the LSTM network of a trained top-level actor in \texttt{Ant-Heaven-Hell} during four episodes with heaven on the left (left figure) and on the right (right figure). The shaded areas mark when the agent is inside the blue area and can observe the side of heaven. }
    \label{fig:cell_evolve}
\end{figure}


\clearpage
\bibliographystyleApp{splncs04}
\bibliographyApp{app_refs}


\end{document}